\documentclass[11pt]{article}
\date{}

\usepackage[utf8]{inputenc} 
\usepackage[T1]{fontenc}    
\usepackage[colorlinks=true,linkcolor=blue,allcolors=blue]{hyperref}      
\usepackage{url}           
\usepackage{booktabs}     
\usepackage{amsfonts}   
\usepackage{nicefrac}       
\usepackage{microtype,nicefrac}     
\usepackage{xspace}
\usepackage{natbib}
\usepackage{wrapfig}
\usepackage{comment}
\usepackage{amsmath}
\usepackage{amsthm}
\usepackage{graphicx}
\usepackage{rotating}
\usepackage{caption}
\usepackage{subcaption}
\usepackage{amssymb}
\usepackage{autobreak}
\usepackage{mathtools}
\usepackage[dvipsnames]{xcolor}
\usepackage{bbm}
\usepackage[ruled,vlined, linesnumbered]{algorithm2e}
\usepackage{algorithmic}
\usepackage[parfill]{parskip}
\usepackage{authblk}
\usepackage{nicefrac}      
\usepackage{microtype}     
\usepackage[dvipsnames]{xcolor}         
\usepackage{float}
\usepackage{amssymb}
\usepackage{mathtools}
\usepackage{amsmath}
\usepackage{amssymb}
\usepackage{mathtools}
\usepackage{amsthm}

\usepackage{scrwfile}

\TOCclone[\contentsname~(\appendixname)]{toc}{atoc}

\AfterTOCHead[toc]{%
}
\AfterTOCHead[atoc]{%
  \edef\maintocdepth{\the\value{tocdepth}}%
  \value{tocdepth}=-10000\relax%
}

\makeatletter

\DeclareMathOperator*{\argmax}{arg\,max}
\DeclareMathOperator*{\argmin}{arg\,min}

\newcommand{\weight}{w}

\newcommand{\FuncClass}{\mathbb{F}}
\newcommand{\GuncClass}{\mathbb{G}}
\newcommand{\Rade}{\mathcal{R}}

\newcommand{\U}{\mathcal{U}}

\newcommand{\CDF}{\textsc{{CDF}}\xspace}

\newcommand{\Real}{\mathbb{R}}
\newcommand{\Borel}{\mathbb{B}}
\newcommand{\Natural}{\mathbb{N}}

\newcommand{\F}{\mathcal F}
\newcommand{\Z}{\mathcal Z}

\newcommand{\X}{\mathcal X}

\renewcommand{\L}{\mathcal L}

\newcommand{\Y}{\mathcal Y}
\newcommand{\V}{\mathcal V}

\newcommand{\E}{\mathbb E}

\newcommand{\Prob}{\mathbb P}
\newcommand{\Ind}{\mathbbm{1}}

\newcommand{\N}{\mathcal N}

\newcommand{\R}{\mathfrak{R}}

\newcommand{\CVaR}{\textnormal{CVaR}}

\newcommand{\wt}{\widetilde}
\newcommand{\wh}{\widehat}
\newcommand{\wb}{\overline}
 
\newcommand{\cvara}{\text{CVaR}_\alpha}
\newcommand{\oce}{\text{oce}}
\newcommand{\ioce}{{\overline{{\text{oce}}}}}
\newcommand{\mybrace}[1]{\left\{#1\right\}}
\newcommand{\JuncClass}{\mathbb{J}}
\newcommand{\HuncClass}{\mathbb{H}}

\makeatletter
\def\moverlay{\mathpalette\mov@rlay}
\def\mov@rlay#1#2{\leavevmode\vtop{%
   \baselineskip\z@skip \lineskiplimit-\maxdimen
   \ialign{\hfil$\m@th#1##$\hfil\cr#2\crcr}}}
\newcommand{\charfusion}[3][\mathord]{
    #1{\ifx#1\mathop\vphantom{#2}\fi
        \mathpalette\mov@rlay{#2\cr#3}
      }
    \ifx#1\mathop\expandafter\displaylimits\fi}
\makeatother

\newtheorem{lemma}{Lemma}[section]
\newtheorem{remark}{Remark}[section]

\newtheorem{definition}{Definition}[section]

\newcommand{\uniformquantity}{e_n(\FuncClass, \ell)}

\usepackage{thmtools, thm-restate}

\newcommand{\leqi}[1]{\textcolor{magenta}{LL:#1}}

\newif\ifarxiv
\arxivfalse

\title{Supervised Learning 
with General Risk Functionals}

\author[1]{Liu Leqi}
\author[2]{Audrey Huang}
\author[1]{Zachary C. Lipton}
\author[3]{Kamyar Azizzadenesheli}
\affil[ ]{\texttt{\href{mailto:leqil@cs.cmu.com}{\textcolor{black}{leqil@cs.cmu.edu}},\href{mailto:audreyh5@illinois.edu}{\textcolor{black}{audreyh5@illinois.edu}},\href{mailto:zlipton@cmu.edu}{\textcolor{black}{zlipton@cmu.edu}},\href{mailto:kamyar@purdue.edu}{\textcolor{black}{kamyar@purdue.edu}}}}

\affil[1]{Machine Learning Department, Carnegie Mellon University}
\affil[2]{Department of Computer Science, University of Illinois Urbana-Champaign}
\affil[3]{Department of Computer Science, Purdue University}

\begin{document}

\maketitle

\begin{abstract}
Standard uniform convergence results
bound the generalization gap 
of the \emph{expected} loss  
over a hypothesis class.
The emergence of risk-sensitive learning 
requires generalization guarantees for functionals of the loss distribution
beyond the expectation.
While prior works specialize in uniform convergence of particular functionals,
our work provides uniform convergence 
for a general class of H\"older risk functionals 
for which the  closeness in the Cumulative Distribution Function (CDF) 
entails closeness in risk.
We establish the first
uniform convergence results 
for estimating the CDF of the loss distribution,
yielding guarantees that hold simultaneously 
both over all H\"older risk functionals
and over all hypotheses. 
Thus licensed to perform empirical risk minimization,
we develop practical gradient-based methods 
for minimizing distortion risks
(widely studied subset of H\"older risks
that subsumes the spectral risks, 
including the mean, conditional value at risk, 
cumulative prospect theory risks, and others)
and provide convergence guarantees.
In experiments, we demonstrate 
the efficacy of our learning procedure, 
both in settings where uniform convergence results hold and in high-dimensional settings with deep networks.

\end{abstract}

\section{Introduction}

To date, the vast majority of supervised, unsupervised, 
and reinforcement learning research has focused on objectives 
expressible as expectations (over some dataset or distribution)
of an underlying loss (or reward) function. 
This focus is understandable.
The expected loss is mathematically convenient
and a reasonable default, 
and a special case of nearly every 
proposed family of risks. 
To be sure, this focus has paid off:
we now possess a rich body of theory and methods
for evaluating, optimizing, and providing
theoretical guarantees on 
{the expected loss.}

However, real-world concerns such as risk aversion, 
equitable allocations of benefits and harms, 
or alignment with human preferences, 
often demand that we address
other functionals of the loss distribution.
For example, in finance, the expectation of returns
must be weighed against their variance
to determine an ideal portfolio allocation,
as codified, e.g., in the mean-variance objective \citep{bjork2014mean}. %
Focusing on supervised learning, 
consider the common scenario
in which a population contains a minority 
(constituting fraction $\alpha$ of the population)
but where group membership 
was not recorded in the available data.
If the pattern relating the features
to the label were different for different demographics,
a naively trained model might adversely 
harm members of a minority group.
Absent further information, 
one sensible strategy could be
to optimize the worst case performance 
over all subsets (of size up to $\alpha$).
This would translate to the familiar 
Conditional Value at Risk (CVaR) objective
\citep{rockafellar2000optimization}.{ In addition, even in settings 
where a model is evaluated in terms of 
the expected loss at test time, 
the training objective may be chosen 
as some other functional
to account for phenomena such as 
distribution shifts~\citep{duchi2018learning}, noisy labels~\citep{lee2020learning}, or imbalanced datasets~\citep{li2020tilted}.
}

Risk-sensitive learning research 
{addresses the problem of learning models 
under many families of (risk) functionals,
}
including (among others) 
distortion risks \citep{wirch2001distortion},
coherent risks \citep{artzner1999coherent},
spectral risks \citep{acerbi2002spectral},
and cumulative prospect theory risks \citep{prashanth2016cumulative}.
Subsuming these risks under a common framework 
addressing bounded losses/rewards,
\citet{huang2021off} recently introduced 
Lipschitz risk functionals,
for which differences in the risk are bounded
by (sup norm) differences in the Cumulative Distribution Function (CDF) of losses.
Thus, because a single CDF estimate 
can be used to estimate all Lipschitz risks,
sup norm concentration of the CDF estimate 
entails corresponding (simultaneous) concentration
of all Lipschitz risks calculated on that CDF estimate. 
However, this concentration result applies 
only to a single hypothesis. 
In contrast, most uniform convergence results 
in learning theory 
have concentrated largely on {the} expected loss \citep{vapnik1999overview,vapnik2013nature,bartlett2002rademacher}. 
While uniform convergence results are known
for several specific risk classes, 
including the spectral/rank-weighted risks~\citep{khim2020uniform} 
and optimized certainty equivalent risks~\citep{lee2020learning}, 
no results to date 
provide uniform convergence guarantees 
that hold simultaneously over 
{both a hypothesis class and a broad class of risks.}

Tackling this problem, 
{we present, to our knowledge, the first uniform convergence guarantee on estimation of the loss CDF.
}
Our bounds rely on appropriate complexity measures of the hypothesis class. 
In addition to relying on the familiar Rademacher complexity and VC dimension,
we propose a new notion of permutation complexity
that is especially suited to \CDF{} estimation. 

For general risk estimation, we adopt the broader class (subsuming the Lipschitz risks) of H\"older risk functionals, for which closeness in distribution entails closeness
in risk. 
Combined with our uniform convergence guarantees for CDF estimation, this property allows us to
establish uniform convergence guarantees for risk estimation of supervised learning models, which hold simultaneously over
all hypotheses in the model class
and over all H\"older risks.

These results license us to optimize general risks,
assuring that for appropriate model classes
and given sufficient data,
the empirical risk minimizer 
will indeed generalize and that 
whichever objective is optimized,
{all} H\"older risk estimates %
will be close to their true values. %
{G}eneralization aside,
optimizing complex risks is non-trivial.
To tackle this problem, we propose a new algorithm 
for optimizing distortion risks, 
{a subset of H\"older risks that subsumes the spectral risks, 
including the expectation, CVaR, 
cumulative prospect theory risks, and others}. 
Our approach extends traditional gradient-based 
empirical risk minimization methods
to handle distortion risks
\citep{denneberg1990distorted,wang1996premium}.
In particular, we calculate
the empirical distortion risk 
by re-weighting losses 
based on \CDF values
and establish convergence guarantees
for the proposed optimization method.
Finally, we experimentally validate our algorithm, 
both in settings where uniform convergence results hold 
and in high-dimensional settings with deep networks.

In summary, we contribute the following:
\begin{enumerate}
     \item 
    The first uniform convergence result for CDF estimation 
    together with corresponding new complexity measures suited to the task
    (Section~\ref{Sec:SV})
    and applications of these results to produce guarantees on risk estimation 
    that hold simultaneously over all H\"older risks
    (Section~\ref{sec:application-risk-assessement}).
    \item 
    A gradient-based method for minimizing distortion risks (widely studied subset of H\"older risks)
    that re-weights examples dynamically based on the empirical \CDF 
    of losses, and corresponding convergence guarantees
    (Section~\ref{sec:optimization}). 
    \item 
    Experiments confirming the practical usefulness of our learning algorithm (Section~\ref{Sec:Exp}).
\end{enumerate}

\allowdisplaybreaks

\section{Related Literature}

Risk functionals have long been studied 
in diverse contexts
\citep{sharpe1966mutual, artzner1999coherent,rockafellar2000optimization,krokhmal2007higher,shapiro2014lectures,acerbi2002spectral,prashanth2016cumulative,jie2018stochastic}. 
CVaR, value-at-risk, and mean-variance
\citep{casselgeneral,sani2013risk,vakili2015mean,zimin2014generalized}
rank among the most widely studied risks.
\citet{prashanth2016cumulative} introduces
the cumulative prospect theory risks,
which have been studied in bandit 
\citep{gopalan2017weighted}
and supervised learning settings~\citep{liu2019human}.
Many previous works have tackled the 
evaluation \citep{huang2021off, chandak2021highconfidence}
and optimization
\citep{torossian2019mathcal,munos2014bandits} 
of risk functionals.

Recent work on risk-sensitive supervised learning 
has established the uniform convergence
of a single risk functional 
when losses incurred by the models are bounded~\citep{khim2020uniform,lee2020learning}, 
or the excess risk of a particular learning procedure 
in cases where the loss
could be unbounded~\citep{holland2021spectral}. 
Collectively, these works have addressed 
the class of spectral risks 
(L-risks or rank-weighted risks) 
that includes the expected value, CVaR and cumulative prospect theory risks~\citep{khim2020uniform,holland2021spectral},
as well as the class of optimized certainty equivalent risks 
that includes the expected value, CVaR and entropic risks \citep{ben1986expected,lee2020learning} .

To our knowledge, the aforementioned risk-sensitive {learning} results are considerably narrower:
the analyses apply only to smaller \emph{families} of risks
and the guarantees hold only for 
a \emph{single risk functional} 
(not simultaneously over the family).
By contrast, we establish uniform convergence results 
that hold simultaneously over both a broader class of risks
and over an entire model class 
(constrained by an appropriate complexity measure).
The key to our approach is to %
{estimate the \CDF of losses and control its sup norm error uniformly over a hypothesis class. }
{CDF estimation} 
is a central topic in learning theory~\citep{devroye2013probabilistic}. 
Strong approximation %
results provide concentration bounds 
on the Kolmogorov–Smirnov distance (sup norm)
between the true and estimated \CDF~\citep{massart1990tight}.
{As our uniform convergence results are over a hypothesis class of possibly infinite number of hypotheses, we control the complexity of the hypothesis class using data-dependent complexity notions (e.g., Rademacher complexity)
and data-independent complexity notions 
(e.g., VC dimension) \citep{alexander1984probability,vapnik2006estimation,ganssler1979empirical}.}

\section{Preliminaries}

We use $\X$ to denote the space of covariates,
$\Y$ the space of labels,
and $\Z = \X \times \Y$. 
Let $\ell:\Y \times \Y \rightarrow \Real$ denote a loss function 
and  $\FuncClass$ a hypothesis class,
where for any $f \in \FuncClass$, $f:\X \rightarrow \Y$.  
The set $\FuncClass_\ell$, with elements $\ell_f$, 
denotes the class of functions 
that are compositions of the loss function $\ell$ 
and a hypothesis $f\in\FuncClass$, 
i.e., $\forall z \in \Z$, $\ell_f(z) = \ell\left(f(x), y\right)$.
Furthermore, we use $\ell_f\left(Z\right)$ 
to denote the random variable of the loss incurred 
by $f \in \FuncClass$ under data $Z = (X,Y)$. 
For any $n\in\Natural$, $[n]:=\lbrace 1,\ldots,n\rbrace$. 

We use $\U$ to denote the space of real-valued 
random variables that admit CDFs. 
For any $U \in \U$, 
its \CDF is denoted by $F_U$. 
A risk functional $\rho: \U \rightarrow \Real$ 
is a mapping from a space of real-valued random variables to reals.
A risk functional is called {law-invariant} 
(or version-independent)
if for any pair of random variables 
$U,U' \in \U$ with the same law ($F_U = F_{U'}$), we have 
$\rho(U)=\rho(U')$~\citep{kusuoka2001law}. 
We work with law-invariant risk functionals 
in this paper, and with some abuse of notation,
we refer to $\rho(F_U)$ and $\rho(U)$ interchangeably.

\section{Uniform Convergence for CDF Estimation 
}\label{Sec:SV}%

We begin with an important building block 
for risk estimation---{\CDF estimation 
with uniform convergence guarantees}. 
Given a loss function $\ell$
and a data set of $n$ labeled data points
$\lbrace Z_i \rbrace_{i=1}^n$ 
where $Z_i =(X_i,Y_i)$,
we are interested in estimating the \CDF of $\ell_f(Z)$ for all $f \in \FuncClass$. 
We use the unbiased empirical \CDF{} estimator: 
\begin{equation}\label{eq:estimator}
    \wh{F}(r;f) := \frac{1}{n}\sum_{i=1}^n \mathbbm{1}_{\{
    \ell_f(Z_i)
    \leq r\}}, 
\end{equation}
where $\E[\wh F(r ; f)] = \Prob(\ell_f(Z) \leq r) = F(r; f)$. 
To establish the uniform convergence of the estimator,
our central goal is to analyze the following quantity:
\begin{align}
    \uniformquantity = \sup_{f\in\FuncClass}\sup_{r\in\Real} \left|\wh{ F}(r;f) -  F(r;f)\right|.
\end{align} 
In Section~\ref{sec:decouple},
we exploit the special structure of CDF estimation 
and propose a new notion of permutation 
complexity 
that captures the complexity 
of the hypothesis class used for \CDF{} estimation. 
In Section~\ref{sec:classical},
we apply the more classical approach 
for analyzing uniform convergence 
that does not exploit any special structure of \CDF estimation. 
Each approach offers a unique perspective
and contributes to our understanding of \CDF estimation. 
We highlight that the uniform convergence we provide
hold for \emph{any} loss distribution
regardless of whether the loss is binary or
bounded. %

{We first introduce notation key to our analysis.}
The Rademacher complexity in our setting 
(for a given loss function $\ell$) is given as follows:
\begin{align} \label{eq:rade-defn}
    \Rade(n,\FuncClass) &= \E_{\Prob,\R}\left[\sup_{f\in\FuncClass}\sup_{r\in\Real} \frac{1}{n}\left|\sum_{i=1}^n\xi_i \Ind_{\mybrace{\ell_f(Z_i)) \leq r}}\right|\right] \nonumber\\
    &= \E_{\Prob,\R}\left[\sup_{f\in\FuncClass}\sup_{g\in\GuncClass(1)} \frac{1}{n}\left|\sum_{i=1}^n\xi_i g(\ell_f(Z_i))\right|\right],
    \quad \raisetag{20pt}
\end{align}
with $\R$ being a Rademacher measure 
on a set of Rademacher random variables $\{\xi_i\}_{i=1}^n$ 
and 
$
    \GuncClass(1) :=\lbrace  \Ind_{\lbrace \cdot \; \leq r \rbrace}: \forall r \in\Real\rbrace
$
is the set of indicator functions parameterized by a real-valued $r$.
Using McDiarmid’s inequality and symmetrization, 
we obtain the following classical result 
that bounds $\uniformquantity$ 
in terms of the Rademacher complexity.
All proofs in this section can be found
in Appendix~\ref{appendix:generalization_proofs}. 

\begin{restatable}{theorem}{thmSvGen}\label{thm:SV_Gen}
    Given a hypothesis class $\FuncClass$,
    any loss function $\ell: \Y \times \Y \to \Real$, 
    and $n$ samples $\{Z_i\}_{i=1}^n$, %
     we have that  with probability at least $1-\delta$, %
    \begin{align*}%
        \uniformquantity 
        \leq 2\Rade(n,\FuncClass) + \sqrt{\frac{\log(\frac{1}{\delta})}{2n}}.
    \end{align*}
\end{restatable}

In general, the Rademacher complexity 
$\Rade(n, \FuncClass)$ is hard to obtain.
Researchers have come up with different ways 
to control it for various hypothesis classes, 
e.g., hypothesis classes with 
finite VC dimension~\citep{wainwright2019high}. 
In the following, we discuss how we work with $\Rade(n, \FuncClass)$.

\subsection{Permutation Complexity}%
\label{sec:decouple}

We first notice that $\Rade(n, \FuncClass)$ 
depends \emph{jointly} on both 
the hypothesis class $\FuncClass$ and $\GuncClass(1)$.
A direct approach that  
follows from the classical statistical learning theory   
is to work with the function class that combines 
$\FuncClass$ and $\GuncClass(1)$, 
which we provide more details in Section~\ref{sec:classical}.
In this section, 
we propose a new way of thinking about $\Rade(n, \FuncClass)$.
By exploiting the special structure of \CDF estimation
(the structure of $\GuncClass(1)$), 
we uncover that $\Rade(n, \FuncClass)$ 
can be controlled by \emph{only} the complexity 
of the hypothesis class $\FuncClass$ 
(or $\FuncClass_\ell$ with elements $\ell_f$).
In order to do so, we first introduce
the notion of \emph{permutation complexity}.  
This complexity measure is data-dependent 
and enables us to work with $\Rade(n, \FuncClass)$
by disentangling the complexity of $\FuncClass$ (or $\FuncClass_\ell$)
from that of $\GuncClass(1)$.

For a measurable space $\V$, let $\zeta:\V\rightarrow\Real$ 
denote a measurable function and $\lbrace v_i\rbrace_{i=1}^n$ 
denote a set of $n$ points in $\V$. 
Satisfying the conditions of selection 
and maximum theorems~\citep[Chapter 17]{guide2006infinite}, 
a permutation function $\pi: [n]\rightarrow [n]$ 
in the space $\Pi(n)$ of all permutation of size $n$ exists, 
and permutes the indices of $\lbrace v_i\rbrace_{i=1}^n$ 
such that $\zeta(v_{\pi(1)})\leq \zeta(v_{\pi(2)})\leq \ldots\leq \zeta(v_{\pi(n)})$. 
We note that the permutation function 
can depend on the specific data points 
$\{v_i\}_{i=1}^n$ and the function $\zeta$ of interests. 
In the following definitions,
we consider a function class $\mathbb{J}$ 
of functions $\zeta:\V\rightarrow\Real$ 
and a probability measure $\mu$ on $(\V,\sigma(\V))$,
where $\sigma(\V)$ denotes the $\sigma$-algebra generated by $\V$.

\begin{definition}[Permutation Complexity]
The instance-dependent permutation complexity of $\mathbb{J}$ 
at $n$ data points $\lbrace v_i\rbrace_{i=1}^n$,
denoted as $\N_\Pi(\mathbb{J},\lbrace v_i\rbrace_{i=1}^n)$, 
is the minimum number of permutation functions $\pi\in\Pi(n)$
needed to sort elements of $\{\zeta(v_i)\}_{i=1}^n$,
$\forall \zeta\in\mathbb{J}$. 
The permutation complexity of $\mathbb{J}$ at $n$ (random) data points $\lbrace V_i\rbrace_{i=1}^n$ with measure $\mu$ is 
\begin{align*}
    \N_\Pi(n,\mathbb{J},\mu):=\E_\mu\left[\N_\Pi(\mathbb{J},\lbrace V_i\rbrace_{i=1}^n)\right].
\end{align*}
\end{definition}
To obtain a better understanding of permutation complexity, 
we provide the permutation complexity of monotone real-valued functions.

\begin{lemma}\label{lemma:PC_monotone}
When $\V$ is the space of reals, 
$\mathbb{J}$ is a set of real-valued 
non-decreasing functions on $\V$, 
the  permutation complexity $\N_\Pi(n,\mathbb{J},\mu) = 1$.
\end{lemma}
\begin{proof}%
    For a set of real-valued points $\lbrace v_i\rbrace_{i=1}^n$, we can construct a permutation function $\pi$ such that $v_{\pi(1)}\leq v_{\pi(2)}\leq \ldots\leq v_{\pi(n)}$. Since all functions in $\JuncClass$ are non-decreasing, for any  $\zeta\in\JuncClass$, we have
    $\lbrace v_i\rbrace_{i=1}^n$ such that $\zeta(v_{\pi(1)})\leq \zeta(v_{\pi(2)})\leq \ldots\leq \zeta(v_{\pi(n)})$. Thus, $\N_\Pi(\mathbb{J},\lbrace v_i\rbrace_{i=1}^n)=1$
    and $\N_\Pi(n,\mathbb{J},\mu) = 1$ for any measure $\mu$. %
\end{proof} 
\begin{remark}
This result implies that the permutation complexity of threshold functions $\GuncClass(1)$ is one. 
\end{remark}

Mapping the definition to our setting, the function class $\JuncClass$ of interest  
is $\FuncClass_\ell$ with elements $\ell_f$.  
The data points $V_i  = Z_i = (X_i, Y_i)$
and the measure $\mu=\Prob$ (the probability measure for $Z$).

An immediate observation is that when $|\FuncClass|$ is finite, 
we only need at most $|\FuncClass|$ permutation functions to sort 
$\{\ell_f(z_i)\}_{i=1}^n$ (one permutation function for each $f \in \FuncClass$), i.e., 
$\N_\Pi(\FuncClass_\ell, \{z_i\}_{i=1}^n) \leq |\FuncClass|$. 
For the special case of binary classification, 
where $\mathcal{Y} = \{0,1\}$
and the loss function $\ell$ is the $0/1$ loss, 
a coarse upper bound on the permutation complexity 
when $\FuncClass$ has finite VC dimension $\nu(\FuncClass)$
is $ \N_\Pi(n,\FuncClass_\ell, \Prob) \leq (n+1)^{\nu(\FuncClass)}$.
This is due to Sauer's Lemma: there are at most  
$(n+1)^{\nu(\FuncClass)}$ ways of labeling the data,
which suggests that we need at most $(n+1)^{\nu(\FuncClass)}$ permutation functions. %
However, as one may have noticed, 
the number of permutation functions needed %
may be (much) smaller than this number. 
For example, consider a binary classification setting 
where we have $3$ data points and $\FuncClass$
is large enough such that all $2^3$ possible losses 
($000, 001, \ldots$) can be incurred. 
In such a case, 
we only need $4$ permutation functions, %
since loss sequences that are non-decreasing (or non-increasing) 
can share the same permutation function, e.g., 
for loss sequences $111, 011, 001, 000$, 
we can use the same permutation function 
$\pi(i) = i, \forall i \in [3]$. 
We note that the permutation complexity 
is defined for not just binary-valued function class{es}. 
Precisely characterizing the permutation complexity 
for different combinations of function classes, 
data distributions and loss functions is of future interest.

\begin{restatable}{theorem}{thmPermutationComplexity}\label{thm:general_class_Rade_permutation_complexity}
    For any hypothesis class $\FuncClass$ and loss function $\ell: \Y \times \Y \to \Real$,  we have that  
    \begin{align*}%
        \Rade(n,\FuncClass)&\leq \sqrt{\frac{\log( 4\N_\Pi(n,\FuncClass_\ell,\Prob))}{2n}}.
    \end{align*}
\end{restatable}
Theorem~\ref{thm:general_class_Rade_permutation_complexity} indicates that, despite 
$\Rade(n, \FuncClass)$ depending on the supremum over both the function class $\FuncClass$ and $\GuncClass(1)$ where $\GuncClass(1)$ is an infinite set, 
$\Rade(n, \FuncClass)$ can be controlled by just the complexity of $\FuncClass_\ell$. 
When the permutation complexity $ \N_\Pi(n,\FuncClass_\ell,\Prob)$ is polynomial in the number of samples $n$, 
we obtain that $\Rade(n, \FuncClass) = {O}(\sqrt{\nicefrac{\log (n)}{n}})$. 
The immediate consequence of Theorem~\ref{thm:general_class_Rade_permutation_complexity}
when the hypothesis class $\FuncClass$ is finite is provided below.

\begin{restatable}{corollary}{thmFiniteFuncClass}\label{thm:finite-Rade}
For a finite hypothesis class $\FuncClass$, %
$$\Rade(n,\FuncClass) \leq  \sqrt{\frac{\log( 4|\FuncClass|)}{2n}}.$$
\end{restatable}
Corollary~\ref{thm:finite-Rade} suggests that  
the generalization bound %
of \CDF estimation %
follows the same rate as classical generalization bound of the expected loss
for finite function classes, i.e., $O(\sqrt{\nicefrac{\log(|\FuncClass|)}{n}})$.

\subsection{Classical Approach}
\label{sec:classical}

In this section, 
we present a more classical approach 
for analyzing the uniform convergence 
without exploiting the specific structure of
\CDF estimation. 
As we have noted before, 
$\Rade(n, \FuncClass)$ depends on 
both $\FuncClass$ and $\GuncClass(1)$.
In this approach, 
we directly work with 
the function class that combines 
$\FuncClass$ and $\GuncClass(1)$: 
for a given loss function $\ell$, we define  
\begin{align*}
    \HuncClass := \lbrace &h: \Z \to \{0,1\} \; : \;
    h(z ; r) = \Ind_{\lbrace 
    \ell_f(z) %
    \leq r \rbrace}, f \in \FuncClass, r \in \Real \rbrace.
\end{align*}
We note that even when $\FuncClass$ is finite, 
$\HuncClass$ is an infinite set.  
The Rademacher complexity~\eqref{eq:rade-defn} can be re-written as 
\begin{align*}
    \Rade(n,\FuncClass) &= \E_{\Prob,\R}\left[\sup_{h \in\HuncClass}\frac{1}{n}\left|\sum_{i=1}^n\xi_i h(Z_i)\right|\right].
\end{align*}

\begin{lemma}\citep[Lemma 4.14]{wainwright2019high}
\label{lemma:naive}
    Let $|\HuncClass({\{z_i\}}_{i=1}^n)|$ denote the maximum cardinality
    of the set $\HuncClass({\{z_i\}}_{i=1}^n) =  
    \left\{(h(z_1), \ldots, h(z_n)): h \in \HuncClass \right\},$ 
    where $n \in \mathbb{N}$ is fixed and $\{z_i\}_{i=1}^n$ can be any data collection for $z_i \in \Z$. 
    Then, we have $$\Rade(n,\FuncClass) \leq  2\sqrt{\frac{\log(|\HuncClass({\{z_i\}}_{i=1}^n)|)}{n}}.$$
\end{lemma}
Since $\HuncClass$ consists of binary functions, for any data collection ${\{z_i\}}_{i=1}^n$, 
the set $\HuncClass({\{z_i\}}_{i=1}^n)$ is finite.  
When $\FuncClass$ is finite, we obtain that 
$|\HuncClass({\{z_i\}}_{i=1}^n)| \leq (n+1)|\FuncClass|$.
This is true since for a given data collection $\{z_i\}_{i=1}^n$
and hypothesis $f \in \FuncClass$, 
after sorting $\{\ell_f(z_i)\}_{i=1}^n$, 
we have at most $(n+1)$ ways of labeling them using the indicator functions 
$\Ind_{\lbrace \cdot \leq r \rbrace}$ for $ r \in \Real$. 
Thus, if we directly apply Lemma~\ref{lemma:naive}, 
we obtain that the Rademacher complexity $\Rade(n, \FuncClass)$ is on the order of $O(\sqrt{\nicefrac{(\log|\FuncClass| + \log n)}{n}})$,
which has an extra $\log(n)$ term in the numerator 
compared to our bound in Corollary~\ref{thm:finite-Rade}. 
In general, when $\HuncClass$ has finite VC dimension $\nu(\HuncClass)$, 
we obtain that $\Rade(n,\FuncClass) \leq %
O(\sqrt{\nicefrac{\nu(\HuncClass)\log(n+1)}{n}})$.
However, this result is far from being sharp.
Using more advanced techniques, e.g., chaining and Dudley's entropy integral~\citep{wainwright2019high},
one can remove the extra $\log(n)$ factor on the numerator, 
and obtain that $\Rade(n, \FuncClass) \leq  O(\sqrt{\nicefrac{\nu(\HuncClass)}{n}})$ %
(for more details, see~\citet[Example 5.24]{wainwright2019high}). 
As a consequence, when $|\FuncClass|<\infty$, 
we have $\nu(\HuncClass) \leq \log(|\FuncClass|)$
and thus obtain that $\Rade(n, \FuncClass) \leq O(\sqrt{\nicefrac{\log (|\FuncClass|)}{n}})$
which is at the same rate as our bound in  Corollary~\ref{thm:finite-Rade}.

\section{Uniform Convergence for H\"older Risk Estimation}
\label{sec:application-risk-assessement}

Using our results in Section~\ref{Sec:SV}, 
we show uniform convergence 
for a broad class of risk functionals---H\"older risks.
As illustrated in Section~\ref{sec:optimization},
uniform convergence for a single risk functional
provides grounding for learning models 
through minimizing the empirical risk. 
In addition to uniform convergence for a single risk,
we also provide uniform convergence results 
hold for a collection of risks simultaneously. 
The second result is important 
due to the following reasons:
Although models are trained to optimize a single risk objective, evaluating their performance under multiple risks can give a holistic assessment of their behavior---a task we call \emph{risk assessment}. For example, models minimizing CVaR at different $\alpha$ levels may have different tradeoffs with their expected loss, and monitoring the progress of both objectives throughout the training process can inform choice of the best model. 
In addition, when given a set of models obtained under different learning mechanisms, 
one may want to compare them in terms of different risks. 
To this end, 
using our results on uniform convergence for CDF estimation (Section~\ref{Sec:SV}), 
we demonstrate how a collection of models may be assessed under many risks simultaneously,
{with estimation errors of the same order as the \CDF estimation error.}

\subsection{H\"older Risk Functionals}\label{sec:risks}
We begin by introducing a new class of risks---the H\"older risk functionals---that includes many popularly studied risks and generalizes the notion of Lipschitz risk functionals~\citep{huang2021off} 
and H\"older continuous functionals in Wasserstein distance~\citep{bhat2019concentration}.

\begin{definition}
Let $d$ denote a quasi-metric\footnote{Quasi-metrics are defined in Appendix~\ref{appendix:holder-risk-functionals}.} on the space of CDFs.
A risk functional $\rho$ is $L(\rho,p,d)$ H\"older
on a space of real-valued random variables $\U$ 
if there exist constants $p>0$
and $L(\rho,p,d) > 0$  such that 
for all $U, U' \in \U$ with \CDF $F_U$ and $F_{U'}$ respectively, 
the following holds:
\begin{align*}
    |\rho(F_U)-\rho(F_{U'})|\leq L(\rho,p,d) d(F_U,F_{U'})^p.
\end{align*}
\end{definition}

The class of H\"older risk functionals 
subsumes many other risk functional classes. 
In particular, Lipschitz risk functionals~\citep{huang2021off}
are $ L(\cdot,1,\L_\infty)$ H\"older 
on bounded random variables.  
As a direct result, 
distortion risk functionals with Lipschitz distortion functions~\citep{denneberg1990distorted,wang1996premium,wang1997axiomatic, balbas2009properties,wirch1999synthesis,wirch2001distortion}, 
cumulative prospect theory  risks~\citep{prashanth2016cumulative,liu2019human}, 
variance, 
and linear combinations of aforementioned  functionals 
are all H\"older on bounded random variables.
In addition%
, %
the %
optimized certainty equivalent risks~\citep{lee2020learning} 
and spectral risks%
~\citep{khim2020uniform,holland2021spectral} %
recently studied in risk-sensitive 
supervised learning literatures, %
{are Lipschitz (hence H\"older)}
on the space of bounded random variables. 
We provide proofs and further details in Appendix~\ref{appendix:holder-risk-functionals}.%

To be more specific, we present a subset of H\"older risk functionals---distortion risks 
with Lipschitz distortion functions---that consists of many well-studied risks
including the expected value, CVaR  and cumulative prospect theory risks.
When the loss is non-negative, 
the distortion risk of $\ell_f(Z)$ is defined to be:
\begin{align}\label{eq:distortion-risk}
    \rho(F(\cdot; f)) = \int_0^\infty g( 1 - F(r; f)) dr,
\end{align}
where the distortion function $g:[0, 1] \to [0,1]$ is non-decreasing 
with $g(0)=0$ and $g(1) =1$. 
In the case of expected value, the distortion function 
$g$ is $1$-Lipschitz.
For CVaR at level $\alpha$, 
the distortion function $g$ is $\frac{1}{\alpha}$-Lipschitz.
For more details, we refer the readers to~\citet{huang2021off}.
In the following, we show uniform convergence for estimating H\"older risks using our proposed estimator (Section~\ref{sec:uniform-convergence-risk-estimation}) 
and develop optimization procedures to minimize distortion risks (Section~\ref{sec:optimization}).

\subsection{Uniform Convergence for Risk Estimation}
\label{sec:uniform-convergence-risk-estimation}

For a given hypothesis $f \in \FuncClass$ and loss function $\ell$, 
we estimate the risk $\rho(F(\cdot; f))$
using the \CDF{} estimator $\wh F(\cdot; f)$  
by plugging it in the functional of interest: 
$
    \rho(\wh F(\cdot; f)). 
$
Many existing risk estimators in the supervised learning literatures, including the traditional empirical risk and estimators used for estimating spectral risks~\citep{khim2020uniform} and optimized certainty equivalent risks~\citep{lee2020learning}
can be viewed as examples of the above estimator.

Leveraging the uniform convergence results of the \CDF estimator, we present uniform convergence 
result of the proposed risk estimator. 
The uniform convergence holds both over the hypothesis class $\FuncClass$
and over a set of H\"older risk functionals. 
As the H\"older class contains a large set of popularly studied risks, our result demonstrates that models can be assessed under many risks without loss of statistical power. This is formalized in Theorem~\ref{Thm:riskbound} below. 

\begin{restatable}{theorem}{thmRiskAssessment}\label{Thm:riskbound}
For a hypothesis class $\FuncClass$, a bounded loss function $\ell$, 
and $\delta \in (0, 1]$,
if  
$\Prob(
\uniformquantity %
\leq \epsilon) \geq 1 - \delta$,  
then with probability $1-\delta$, for all $\rho \in \mathbb{T}$, we have 
\begin{align*}
\sup_{f \in \FuncClass}|\rho(F(\cdot;f))-\rho(\wh F(\cdot;f))|\leq  L(\rho,1,\L_\infty) \epsilon,
\end{align*}
where $\mathbb{T}$ is the set of $L(\cdot, 1, \L_\infty)$ H\"older risk functionals 
on the space of bounded random variables. 
\end{restatable}

In cases where the CDF estimation error $\epsilon$ is of order $O(1/\sqrt{n})$, 
we can 
estimate 
the set of  H\"older risks $\mathbb{T}$ for all hypotheses in $\FuncClass$
at rate $O(1/\sqrt{n})$. %
{Because our result is uniform over both the hypothesis class $\FuncClass$ and the risk functional class $\mathbb{T}$, it is} a generalization
of existing 
uniform convergence results that are uniform over $\FuncClass$, but for a single risk functional~\citep{khim2020uniform,lee2020learning}.

In Appendix~\ref{appendix:risk-assessment}, 
we provide similar uniform convergence results
where the set of risk functionals are H\"older smooth in Wasserstein distance. 
As a direct consequence to Theorem~\ref{Thm:riskbound}, 
uniform convergence of a single H\"older risk, 
e.g., a distortion risk~\eqref{eq:distortion-risk},
is given below.

\begin{restatable}{corollary}{corDistortionRiskAssessment}\label{corollary:distortion-risk-uniform-convergence}
For a hypothesis class $\FuncClass$, a bounded loss function $\ell: \Y \times \Y \to [0, D]$, 
and $\delta \in (0, 1]$,
if  
$\Prob(
\uniformquantity%
\leq \epsilon) \geq 1 - \delta$, 
then with probability $1-\delta$, we have 
\begin{align*}
\sup_{f \in \FuncClass}|\rho(F(\cdot;f))-\rho(\wh F(\cdot;f))|\leq  L \epsilon,
\end{align*}
where $\rho$ is a distortion risk with  
$\frac{L}{D}$-Lipschitz distortion function. 
\end{restatable}

\begin{remark}
As an example, consider a binary classification setting 
where the loss function $\ell$ 
is the $0/1$ loss and 
the hypothesis class $\FuncClass$ is finite, 
using our results in Corollary~\ref{corollary:distortion-risk-uniform-convergence}, 
we obtain that the the generation error
for the expected value
is $O(\sqrt{\nicefrac{\log (|\FuncClass|)}{n}})$
and the generation error for the CVaR 
is $O(\sqrt{\nicefrac{\log (|\FuncClass|)}{\alpha^2 n}})$, 
which are of the same rates (with better constants)
as the ones in
\citet{lee2020learning}.
\end{remark}

\section{Empirical Risk Minimization }\label{sec:optimization}

For a single risk functional, 
one may want to learn models that optimize it. 
Our uniform convergence 
results for risk estimation  license us to learn models that minimize the population risk
through Empirical Risk Minimization (ERM). 
We denote the population 
and empirical risk minimizers as follows: 
\begin{align}\label{eq:minimizer}
    f_\star \in \argmin_{f \in \FuncClass} \rho(F(\cdot; f)), \; 
    \wh f_\star \in \argmin_{f \in \FuncClass} \rho(\wh F(\cdot; f)).
    \quad \raisetag{20pt}
\end{align} 
The excess risk of the empirical risk minimizer $\wh f_\star$ can be bounded by 
\begin{align*}
    &\rho(F(\cdot; \wh f_\star)) - \rho(F(\cdot; f_\star))\\
    = \; & \rho(F(\cdot; \wh f_\star)) - \rho(\wh F(\cdot; \wh f_\star))
    + \rho(\wh F(\cdot; \wh f_\star)) - \rho(\wh F(\cdot; f_\star))
    + \rho(\wh F(\cdot; f_\star)) 
     - \rho(F(\cdot; f_\star))\\
    \leq \; & 2 \sup_{f \in \FuncClass} |\rho(\wh F(\cdot; f)) - \rho(F(\cdot;f))|.
\end{align*}

We study ERM when 
the loss function is non-negative
and the risk functional of interest 
is a distortion risk with a Lipschitz distortion function~\eqref{eq:distortion-risk}. 
Such distortion risk functionals consist many well-studied risks,  
including the expected value, CVaR, cumulative prospect theory risks, and other spectral risks~\citep{bauerle2021minimizing}.  
Using Corollary~\ref{corollary:distortion-risk-uniform-convergence}, 
we obtain that 
when $
\uniformquantity %
= O({1/\sqrt{n}})$ 
and the loss $\ell$ is bounded, 
the excess risk of $\wh f_\star$ is $O(1/\sqrt{n})$.

We consider settings where the hypothesis class $\FuncClass$ is a class of parameterized functions, e.g., linear models and neural networks and use $\Theta \subseteq \Real^d$ 
to denote the set of parameters.  
For a hypothesis $f \in \FuncClass$ parameterized by $\theta \in \Theta$, 
we denote  $F(\cdot; f)$ and  $\ell_f(z_i)$
by $F_\theta$ and $\ell_\theta(i)$,
respectively.  
Similarly, we use $\theta_\star$ and $\wh \theta_\star$
for referring to $f_\star$ and $\wh f_\star$. 
As in Section~\ref{sec:decouple},
we use $\pi_\theta:[n] \to [n]$
to denote the permutation function 
such that $\ell_\theta(\pi_\theta(i))$
is the $i$-th smallest loss under 
the current model $\theta$ 
and the fixed dataset $\{z_i\}_{i=1}^n$. 
Using the \CDF estimator $\wh F_\theta$~\eqref{eq:estimator}, 
the \emph{empirical distortion risk} $\rho(\wh F_\theta)$ can be re-written as 
\begin{align*}
      \sum_{i=1}^{n} 
      g\left(1 - \frac{i-1}{n}\right) \cdot \left(\ell_\theta(\pi_\theta(i)) - \ell_\theta(\pi_\theta(i-1))\right),  
\end{align*}
where for all $\theta \in \Theta$,
we set $\ell_\theta(\pi_\theta(0)) := 0$ 
since the losses are non-negative.

To employ first-order methods for minimizing the empirical distortion risk, 
it is natural to first identify when $\rho(\wh F_\theta)$ is differentiable. 

\begin{restatable}{lemma}{thmDifferentiableAE}\label{thm:almost-differentiable}%
    If $\{\ell_\theta(z_i)\}_{i=1}^n$
    are Lipschitz continuous in $\theta \in \Theta$, 
    then for all $i \in [n]$, $\ell_\theta(\pi_\theta(i))$, i.e., 
    the $i$-th smallest loss, is %
    Lipschitz continuous
    in $\theta$
    and $\rho(\wh F_\theta)$ is 
    differentiable in $\theta$ almost everywhere.
\end{restatable}

When $\rho(\wh F_\theta)$
(and $\ell_\theta(\pi_\theta(i))$) is differentiable, 
the gradient $\nabla_\theta  \rho(\wh F_\theta)$ can be written as
\begin{align}\label{eq:full-gradient}
    \sum_{i=1}^{n} \left(g\left(1-\frac{i-1}{n}\right) - g\left(1 -\frac{i}{n}\right)\right) \cdot \nabla_\theta \ell_{\theta}(\pi_\theta(i)). 
    \raisetag{10pt}
\end{align}

When $g(x) = x$, the distortion risk 
is the same as the expected loss 
and we recover the gradient 
for the traditional empirical risk.

To avoid the non-differentiable points,  
we add a small noise to the gradient descent steps.  %
By doing so, we ensure that 
the  descent steps will end up in differentiable points
almost surely.
Choose initial point $\theta_1 \in \Theta$.
At iteration $t$, the parameter is updated as follows
\begin{align}\label{eq:gd-step}
    \theta_{t+1} \gets \theta_t - \eta \left(
    \nabla_\theta \rho (\wh F_\theta)+ w_t \right), 
\end{align}
where 
$\eta$ is the learning rate, 
$\nabla_\theta \rho (\wh F_\theta)$
is given in~\eqref{eq:full-gradient}
and $w_t$ is sampled 
from a $d$-dimensional Gaussian 
with mean $0$ and variance $\frac{1}{d}$.

In general, even when the loss function $\ell$ is convex in the parameter $\theta$,
the empirical distortion risk $\rho(\wh F_\theta)$
may not be convex in $\theta$. 
In Corollary~\ref{lemma:biased_stochastic_gradient}, 
we show local convergence of $\theta_t$
obtained through following~\eqref{eq:gd-step}.

\begin{restatable}{corollary}{corEDRM}\label{lemma:biased_stochastic_gradient}
If $\{\ell_\theta(z_i)\}_{i=1}^n$ are Lipschitz continuous and 
$\rho(\wh F_\theta)$ %
is $\beta$-smooth in $\theta$,
then 
the following holds almost surely
when the learning rate in~\eqref{eq:gd-step} is $\eta = \frac{1}{\beta \sqrt{T}}$: 
\begin{align*}
   \frac{1}{T} \sum_{t=1}^T\E\left[ \|\nabla_\theta \rho(\wh{F}_{\theta_t})\|^2\right]
   \leq \frac{2\beta }{\sqrt{T}} \left(\rho(\wh F_{\theta_1}) - \rho(\wh F_{\theta_\star})  + \frac{1}{2\beta}\right).
\end{align*}
\end{restatable}

Corollary~\ref{lemma:biased_stochastic_gradient} 
demonstrates that 
the average gradient magnitude over $T$ iterations  
shrinks as $T$ goes to infinity,
suggesting that 
performing gradient descent by 
following~\eqref{eq:gd-step} 
will converge to an approximate stationary point.

\section{Experiment}\label{Sec:Exp}

In our experiments, 
we demonstrate the efficacy of  our proposed 
estimator for risk estimation 
and proposed learning procedure for obtaining risk-sensitive models.
In Section~\ref{sec:expt-risk-imagenet},
we work with a risk assessment setting 
where we simultaneously inspect a finite set of models
in terms of multiple risks. 
In Section~\ref{sec:expt-edrm}, 
we show the performance of empirical risk minimization 
under various distortion risks. 
After showing that the classifier learned under different 
risk objectives behave differently in a toy example,
we learn risk-sensitive models for CIFAR-10.

\subsection{Risk Assessment on ImageNet Models}
\label{sec:expt-risk-imagenet}
We perform risk assessments on pretrained Pytorch models 
for ImageNet classification. 
In particular, we choose models with similar accuracy 
(both reported on the official Pytorch website 
and confirmed by us)
on the validation set (50,000 images) 
for the ImageNet classification challenge~\citep{ILSVRC15}. 
The models are VGG-11~\citep{simonyan2014vgg}, GoogLeNet~\citep{szegedy2015googlenet}, ShuffleNet~\citep{ma2018shufflenet},
Inception~\citep{szegedy2016inception} and ResNet-18~\citep{he2016resnet} and
the accuracy of these models evaluated on the validation set are around $69\%$ (Table~\ref{tbl:risk-assessments}). 
By assessing the risks of models with similar accuracy, 
we highlight how models with similar performance 
under traditional metrics (e.g., accuracy) 
could have different risk performances. 
For example, though Inception has similar accuracy
compared to other models, 
its loss variance is much higher compared to others, 
which may be detrimental in settings 
where high-varying performance is not preferred.
We also evaluated the CVaR 
of these models under different $\alpha$'s
(Figure~\ref{fig:cvar-evaluation}
in Appendix~\ref{appendix:experiment}).
Our theoretical results suggest 
that all these evaluations hold {simultaneously} 
across the risk functionals and models of interest
with the error being 
$O(\sqrt{\nicefrac{\log |\FuncClass|}{n}})$
($|\FuncClass| = 5$ in this experiment).
In addition to showcasing 
the power of our theoretical results, this example demonstrates how model assessments under multiple risk notions provide a better understanding of model behaviors.

\begin{table*}%
   \centering
   \begin{tabular}{@{}cccccc@{}}
   \toprule
   {}                 & {{VGG-11}} & {{GoogLeNet}} & {{ShuffleNet}} & {{Inception}} & {{ResNet-18}} \\ \midrule
   {Accuracy}              & {69.022\%}       & {69.772\%}           & {69.356\%}            & {69.542\%}           & {69.756\%}          \\ 
   {$\E[\ell_f]$}               & {1.261}       & {1.283}           & {1.360}            & {1.829}           & {1.247}          \\
   {$\CVaR_{.05}(\ell_f)$}        & {1.327}       & {1.350}           & {1.431}            & {1.925}           & {1.313}          \\ 
   {$\E[\ell_f] + 0.5 \text{Var}(\ell_f)$} & {5.215}       & {4.376}           & {6.718}            & {14.416}          & {5.353}          \\
   {$\text{HRM}_{.3, .4}(\ell_f)$}      & {1.374}       & {1.336}           & {1.542}            & {2.214}           & {1.382}          \\
   {$\text{HRM}_{.2, .8}(\ell_f)$}      & {1.233}       & {1.239}           & {1.344}            & {1.845}           & {1.225}          \\ \bottomrule
   \end{tabular}
   \caption{Risks for different ImageNet classification models evaluated on the validation set. $\ell_f(Z)$ is the cross-entropy loss for each model $f$. For simplicity, we omitted the arguments $Z$ in the table. $\CVaR_\alpha$ is the expected value of the top $100\alpha$ percent losses. %
   $\text{HRM}_{a,b}$ is the cumulative prospect theory risk defined in~\citet{liu2019human}. All results are rounded to $3$ digits.}
   \label{tbl:risk-assessments} 
   \end{table*}

\subsection{Empirical Distortion Risk Minimization}
\label{sec:expt-edrm}

To illustrate the difference among models
learned under different risk objectives,
we first present a toy example 
for comparing models learned under the expected loss objective  
and the CVaR objective respectively.  
We then show the efficacy of our proposed
optimization procedure through  
training deep neural networks on CIFAR-10.
In both cases, 
the models are learned by following~\eqref{eq:gd-step}. 
For more details on these experiments, 
we refer the readers to Appendix~\ref{appendix:experiment}.

\paragraph{Toy Example} In the toy example, 
we work with a binary classification task 
where the covariates are $2$-dimensional. 
In Figure~\ref{fig:2d_illustration}, 
the blue pluses and orange dots represent two classes, 
respectively. 
We have learned logistic regression models 
to minimize the expected loss 
and the $\CVaR_{.05}$ (expected value of the top $5\%$ losses) through minimizing their empirical risks. 
The loss distribution 
along with 
the prediction contours of the two classifiers
showcase the difference between the two models. 
In particular, 
the model learned under expected loss 
suffers high loss for a small subset of the covariates
while 
the model learned under $\CVaR_{.05}$ have 
all losses concentrated around a small value.
Indicated by the (uniform) grey color in the contour plot,
the predictions (predicted probability of a covariate being labeled as $1$) 
for the $\CVaR_{.05}$ model are around $0.5$. 
In contrast, 
the predictions for the expected loss model spread 
across a wide range 
between $0$ and $1$.

\begin{figure*}
    \centering
    \includegraphics[width=\linewidth]{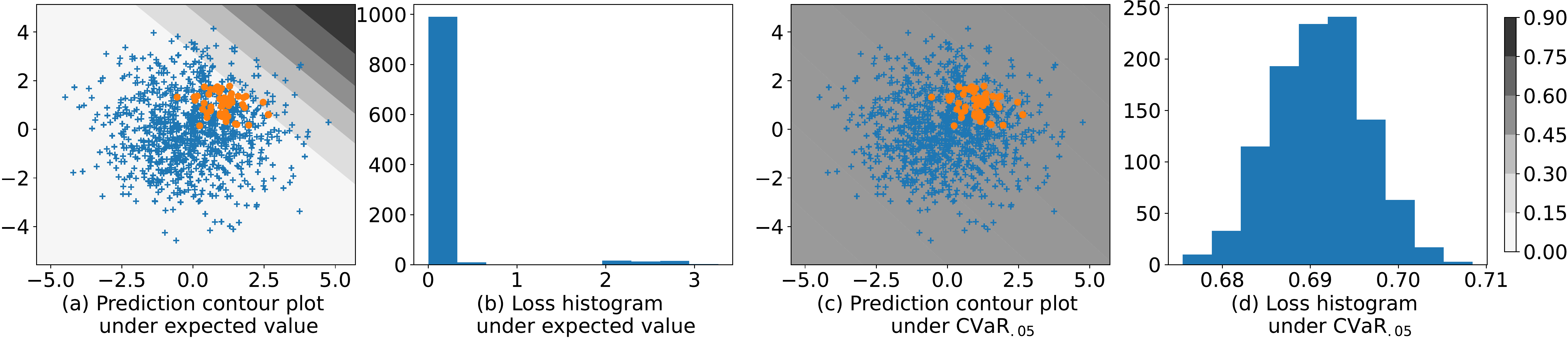}
    \caption{Prediction (predicted probability of a covariate being labeled as $1$) contours and loss histograms of two models learned under the expected loss and the $\CVaR_{.05}$  objective, respectively. 
    The blue pluses and orange dots represent two classes. 
    The loss distribution for the expected loss model has extremely high values for a small subset of the covariates.
    }
    \label{fig:2d_illustration}
\end{figure*}

\begin{figure*}
    \centering
    \begin{subfigure}[b]{0.23\textwidth}
         \centering
         \includegraphics[width=\textwidth]{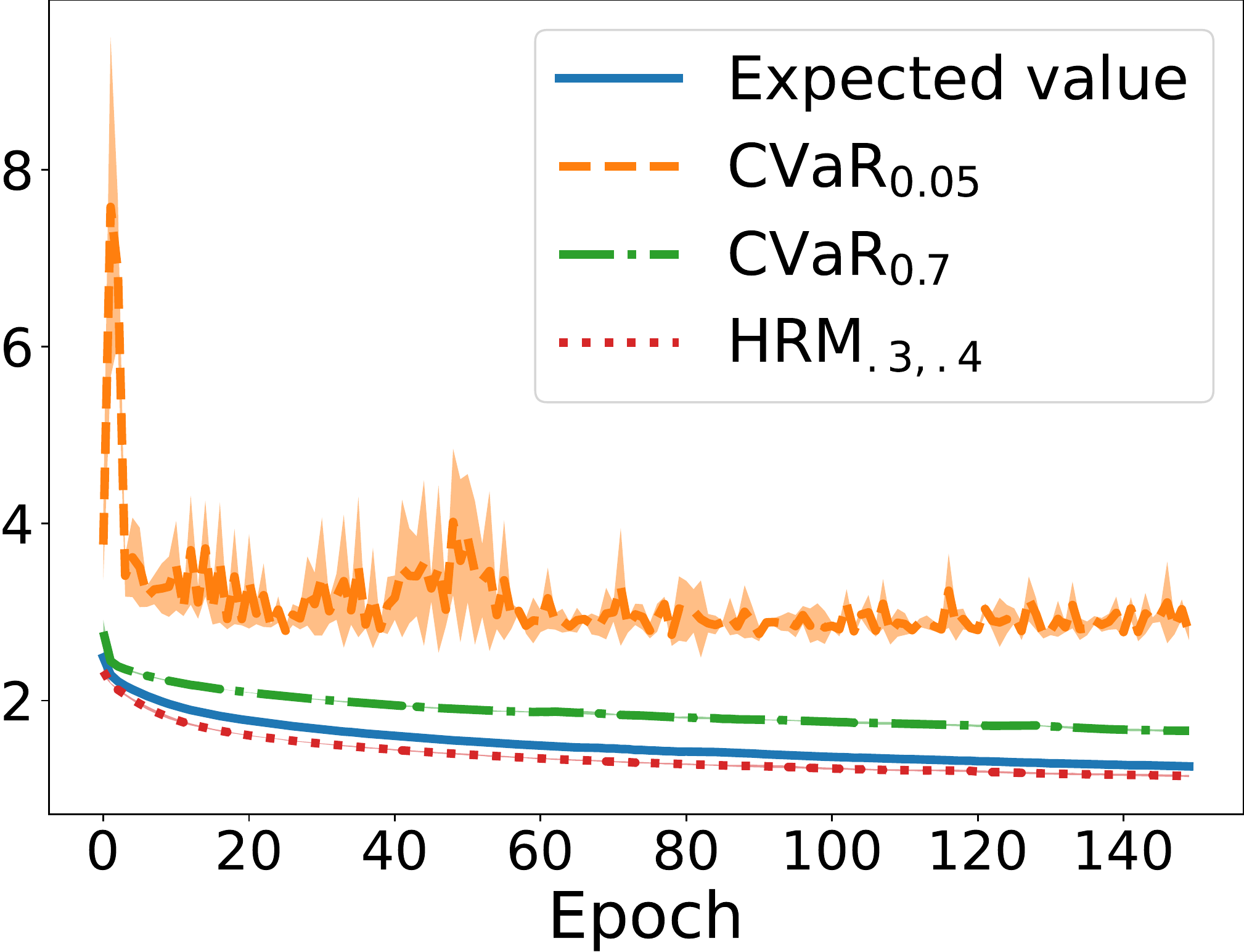}
         \caption{Training objective values}
         \label{fig:train_risk}
     \end{subfigure}
     \hfill
     \begin{subfigure}[b]{0.23\textwidth}
         \centering
         \includegraphics[width=\textwidth]{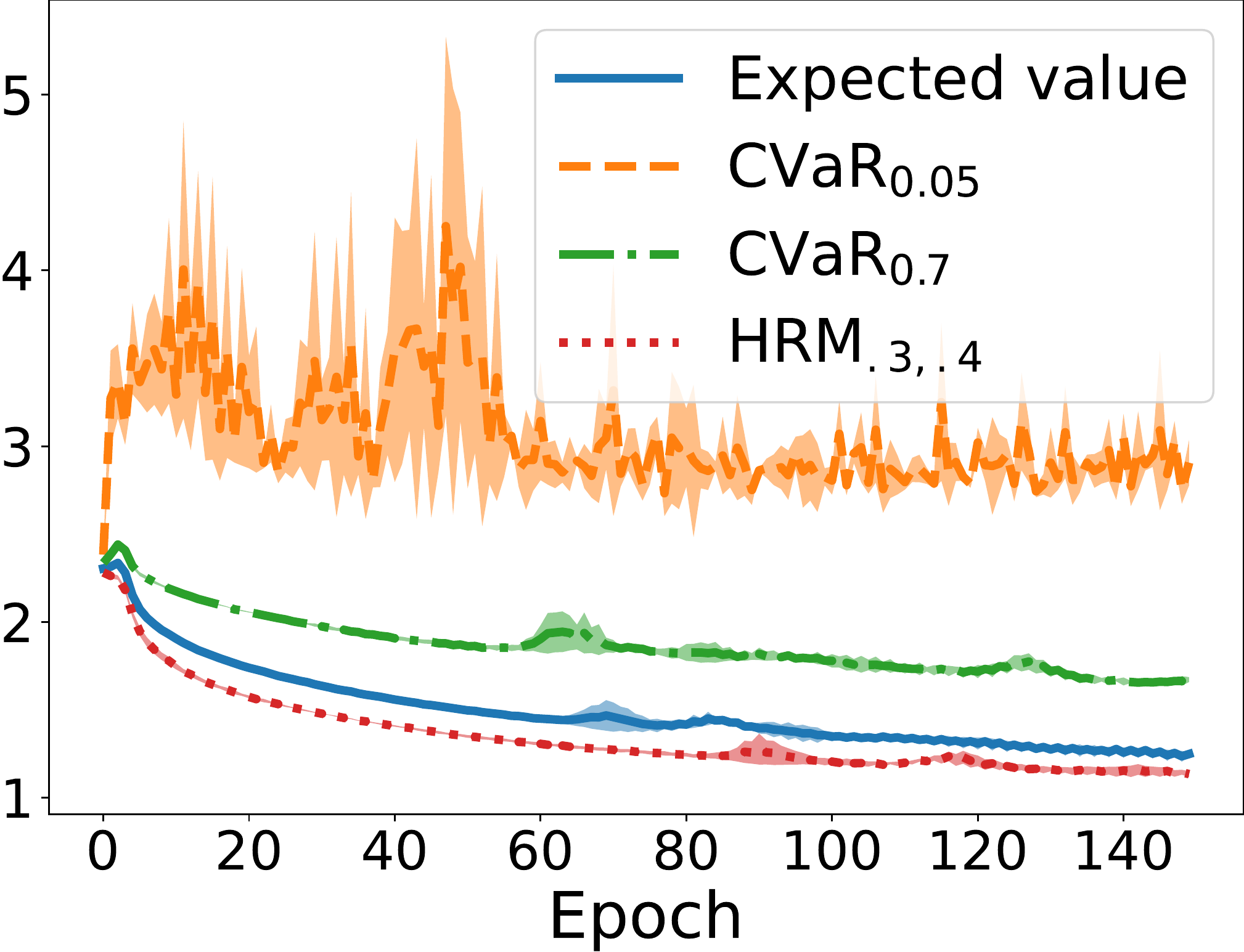}
         \caption{Testing objective values}
         \label{fig:test_risk}
     \end{subfigure}
     \hfill
     \begin{subfigure}[b]{0.23\textwidth}
         \centering
         \includegraphics[width=\textwidth]{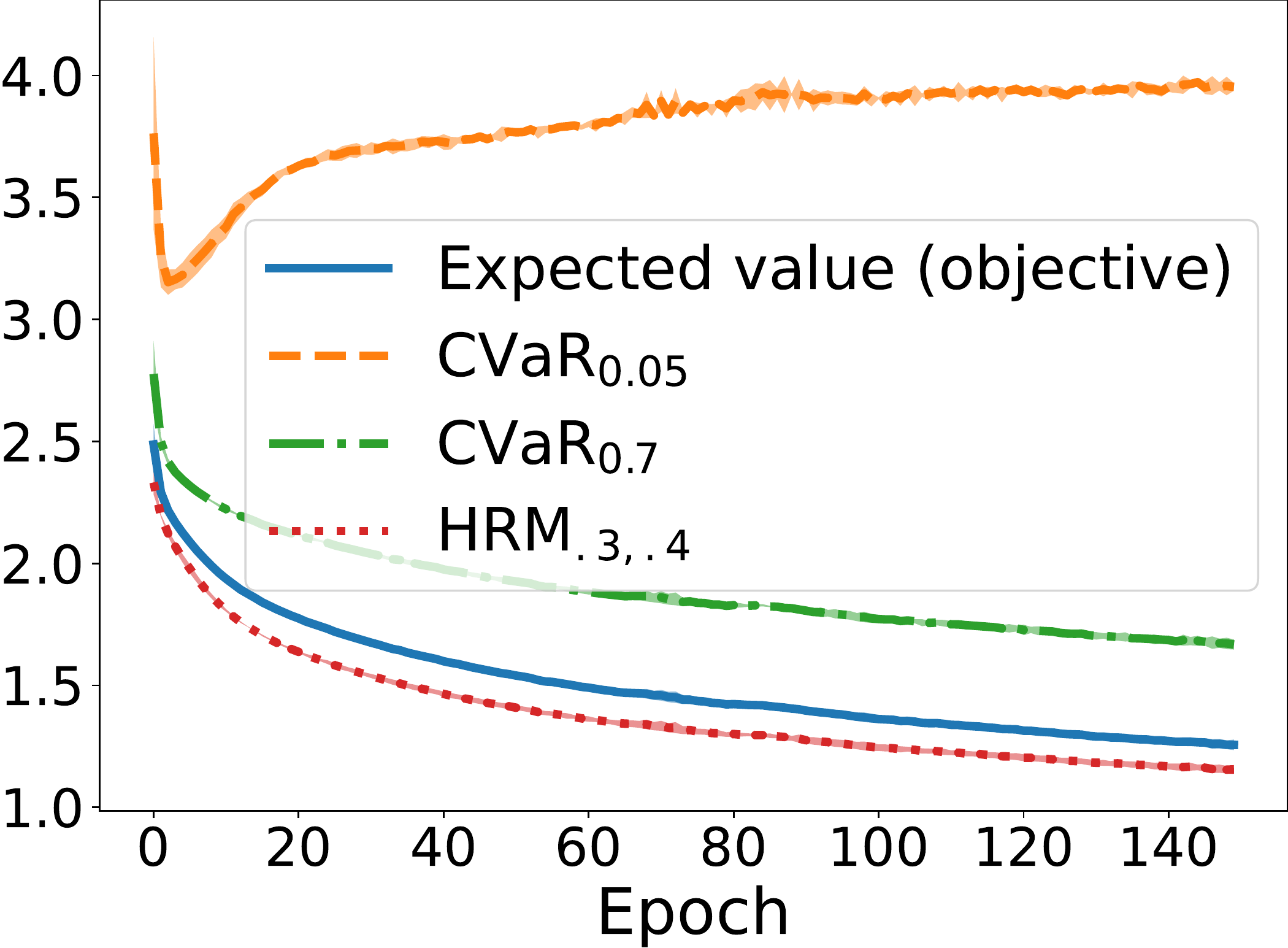}
         \caption{Training risk evaluations}
         \label{fig:ev_other_risks_train}
     \end{subfigure}
     \hfill
     \begin{subfigure}[b]{0.23\textwidth}
         \centering
         \includegraphics[width=\textwidth]{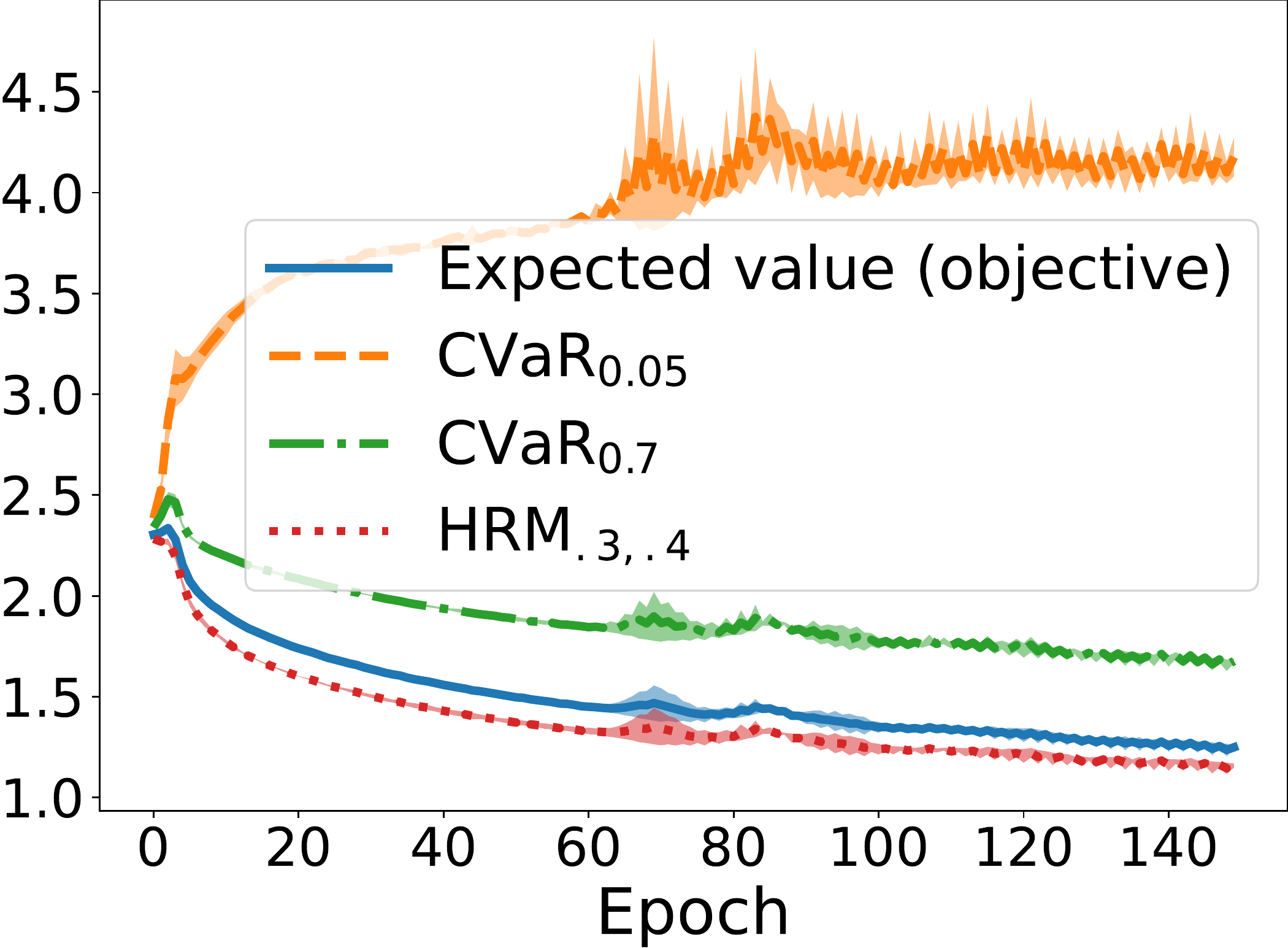}
         \caption{Testing risk evaluations}
         \label{fig:ev_other_risks_test}
     \end{subfigure}
    \caption{
    Performance of VGG-16 models 
    trained under expected loss, $\CVaR_{.05}$, $\CVaR_{.7}$
    and $\text{HRM}_{.3, .4}$.
    In Figures \ref{fig:train_risk} and \ref{fig:test_risk},
    each model is trained and evaluated on the same objective.
    In Figures \ref{fig:ev_other_risks_train} and \ref{fig:ev_other_risks_test},
    we only train one model under the expected loss but report all four objectives of that model.
    All results are averaged over $5$ runs. 
    }
    \label{fig:edrm_cifar10}
\end{figure*}

\paragraph{CIFAR-10}
We have trained VGG-16 models on CIFAR-10
through minimizing the empirical risks
for expected loss, $\CVaR_{.05}$, $\CVaR_{.7}$
and $\text{HRM}_{.3, .4}$~\citep{liu2019human}
using the gradient descent step presented in~\eqref{eq:gd-step}.
The models are trained 
over $150$ epochs
and the learning rate is chosen to be $0.005$. 
As shown in Figure~\ref{fig:train_risk} and \ref{fig:test_risk},
in general,  
the objective values %
are  decreasing 
over the epochs during training and testing.
In addition, we observe that minimizing 
the empirical risk for expected loss 
does not necessarily imply minimizing other risks, 
e.g., $\CVaR_{.05}$ (Figure~\ref{fig:ev_other_risks_train}).
These results suggest the efficacy of  
our proposed optimization procedure 
for minimizing distortion risks.

\vspace{-1em}
\section{Discussion}\label{sec:discussion}

We have presented a principled framework,
including analytic tools and algorithms for  risk-sensitive learning and assessment that:
(1) obtains the empirical CDF; 
(2) estimates the risks of interest through 
plugging in the empirical CDF;
and (3) minimizes the empirical risk (for risk-sensitive learning). 
Our theoretical results on the uniform convergence of the proposed risk estimators hold 
simultaneously over a hypothesis class (constrained by an appropriate complexity measure)
and over H\"older risks.
The key building block for these results is the uniform convergence of the CDF estimator.  

There are multiple future directions of our work. 
First, we hope to 
more precisely characterize 
the permutation complexity 
(under various hypothesis classes). 
Second, our gradient descent procedure~\eqref{eq:gd-step}
requires sorting all losses. 
An important next step would be 
to allow minibatches 
(sorting only a small subset of losses) 
when minimizing empirical distortion risks.
Third, as shown in Figure~\ref{fig:2d_illustration}, 
models learned under different risk objectives 
behave distinctly.
Characterizing these model behaviors theoretically and empirically,
and understanding the trade-offs among these objectives
is crucial for building future models.

\section*{Acknowledgements}
LL is generously supported by an Open Philanthropy AI Fellowship.

\bibliographystyle{plainnat}
\bibliography{risk_supervised_citations}

\clearpage

\appendix

\section{H\"older Risk Functionals}
\label{appendix:holder-risk-functionals}

In the definition of H\"older risk functionals, 
we require $d$ to be a quasi-metric,
which we provide the definition here. 

\begin{definition}
A function $d:\L_\infty(\Real, \Borel(\Real))\times \L_\infty(\Real, \Borel(\Real))\rightarrow [0,+\infty)$ 
is a quasi-metric if the following two conditions hold:
\begin{itemize}
    \item For all $F_U, F_{U'} \in \L_\infty(\Real, \Borel(\Real))$, $d(F_U, F_{U'}) = 0$ if and only if $F_U = F_{U'}$;
    \item For all $F_U, F_{U'},  F_{Z''}\in \L_\infty(\Real, \Borel(\Real))$, 
    $d(F_U, F_{Z''}) \leq d(F_U, F_{U'}) + d(F_{U'}, F_{Z''})$.
\end{itemize}
\end{definition}
If a quasimetric is symmetic, i.e., 
for all $F_U, F_{U'} \in \L_\infty(\Real, \Borel(\Real))$,
$d(F_U, F_{U'}) = d(F_{U'}, F_U)$, 
it is also a \emph{metric}. 
The set of quasi-metrics contains 
symmetric quasi-metics, 
e.g., sup norms $\L_\infty$, Wasserstein distance, 
along with %
non-symmetric quasi-metrics, e.g., Kullback-Leibler divergence.

We will now discuss 
why optimized certainty equivalent (OCE) risks (e.g., mean-variance, entropic risk, \CVaR)
and spectral risks with bounded spectrum (e.g., \CVaR, certain CPT-inspired Risks) are Lipschitz on bounded random variables.
OCE risks,  
first introduced by~\citet{ben1986expected}, are defined as 
\begin{align*}
    \rho_\oce(F_U) := \inf_{\lambda \in \Real}\left\{ \lambda 
    + \mathbb{E}[\phi(U - \lambda)] \right\},
\end{align*}
where 
$\phi:\Real \to \Real \cup \{+\infty\}$ is a nondecreasing, closed and convex function with $\phi(0)=0$
and $1 \in \partial \phi(0)$. 
To complement the risk-averse OCEs, 
a risk-seeking version (inverted OCE) is proposed: 
\begin{align*}
       \rho_\ioce(F_U) := \sup_{\lambda \in \Real}\left\{ \lambda 
    - \mathbb{E}[\phi(\lambda - U)] \right\}.
\end{align*}

\begin{restatable}{proposition}{propOCE}\label{prop:oce_lipschitz}
If $\phi$ is continuously differentiable,
then the OCE risks $\rho_\oce$ and inverted OCE risks $\rho_\ioce$ are Lipschitz 
on the space of bounded random variables with support $[0, D]$:
\begin{align*}
    |\rho_\oce(F_U) - \rho_\oce(F_{U'})|
    &\leq \max_{x \in [0,D]} (\phi(D-x) - \phi(-x))\|F_U - F_{U'}\|_\infty,\\
    |\rho_\ioce(F_U) - \rho_\ioce(F_{U'})|
    &\leq \max_{x \in [0,D]} (\phi(x-D) - \phi(x))\|F_U - F_{U'}\|_\infty. 
\end{align*}
\end{restatable}

\begin{remark}
Similar to~\citet[Lemma 4.1]{huang2021off}, Proposition~\ref{prop:oce_lipschitz} shows that the expected value and $\cvara$ are $D$- and $\frac{D}{\alpha}$-Lipschtiz on random variables with support $[0,D]$, respectively. 
In addition, this result also provides Lipschitzness of the entropic risks  (since the corresponding $\phi$ is continuously differentiable~\citep[Table 1]{lee2020learning})
and other OCE and inverted OCE risks that do not belong to distortion risk functionals.
\end{remark}
\begin{proof}%
When $U$ has support $[0,D]$, 
as shown in~\citet[Lemma 9]{lee2020learning}, 
we can re-write the OCE and inverted OCE risks as follows:
\begin{align*}
    \rho_\oce(F_U) &= \min_{\lambda \in [0,D]}\left\{ \lambda 
    + \mathbb{E}[\phi(U - \lambda)] \right\}\\
    \rho_\ioce(F_U) &= \max_{\lambda \in [0, D]}\left\{ \lambda 
    - \mathbb{E}[\phi(\lambda - U)] \right\}.
\end{align*}

For any $U, U' \in [0,D]$,
denote $\lambda_U \in \argmin_{\lambda \in [0,D]} \lambda + \mathbb{E}_{U}[\phi(U-\lambda)]$ and
$\lambda_{U'} \in \argmin_{\lambda \in [0,D]} \lambda + \mathbb{E}_{U'}[\phi(U'-\lambda)]$. 
Consider the case where $\rho_\oce(F_U) < \rho_\oce(F_{U'})$. 
\begin{align*}
    |\rho_\oce(F_U) - \rho_\oce(F_{U'})| &= \rho_\oce(F_{U'}) - \rho_\oce(F_U)\\ 
    &= \lambda_{U'} + \mathbb{E}_{U'}[\phi(U'-\lambda_{U'})]
    - \lambda_{U} - \mathbb{E}_{U}[\phi(U-\lambda_{U})]\\
    &\stackrel{(i)}{\leq} \lambda_U + \mathbb{E}_{U'}[\phi(U'-\lambda_U)]
    - \lambda_{U} - \mathbb{E}_{U}[\phi(U-\lambda_{U})]\\
    &= \int_0^D \phi(u-\lambda_U) d\left( F_{U'}(u) - F_U(u) \right)\\
    &\stackrel{(ii)}{=} \phi(u-\lambda_U) \left(F_{U'}(u) - F_U(u) \right) \Big|_{u=0}^{u=D} - \int_0^D \phi'(u-\lambda_U) \left( F_{U'}(u) - F_U(u)\right)du\\
    &\stackrel{(iii)}{\leq}
    \|F_U-F_{U'}\|_\infty \int_0^D \phi'(u-\lambda_U) du \\
    &=%
    \left(\phi(D - \lambda_U) - \phi(-\lambda_U) \right) \|F_U-F_{U'}\|_\infty,
\end{align*}
where $(i)$ comes from the definition of $\lambda_{U'}$,
$(ii)$ uses integration by parts, and
$(iii)$ uses the fact that $\phi$ is non-decreasing, i.e., 
$\phi'$ is non-negative,
and $F_U(0) = F_{U'}(0)=0, F_U(D) = F_{U'}(D) = 1$.
The case when $\rho_\oce(F_{U'}) < \rho_\oce(F_U)$ proceeds similarly:
\begin{align*}
    |\rho_\oce(F_U) - \rho_\oce(F_{U'})| &= \rho_\oce(F_U) - \rho_\oce(F_{U'})\\ 
    &\leq \int_0^D \phi(u-\lambda_{U'}) d\left( F_U(u) - dF_{U'}(u) \right)\\
    &= \phi(u-\lambda_{U'}) \left(F_U(u) - F_{U'}(u) \right) \Big|_{u=0}^{u=D} - \int_0^D \phi'(u-\lambda_{U'}) \left( F_U(u) - F_{U'}(u)\right)du\\
    &\leq %
    \left(\phi(D - \lambda_{U'}) - \phi(-\lambda_{U'}) \right) \|F_U-F_{U'}\|_\infty.
\end{align*}
Putting it together, we have that 
\begin{align*}
    |\rho_\oce(F_U) - \rho_\oce(F_{U'})|
    \leq \max_{\lambda \in [0,D]} (\phi(D-x) - \phi(-x))\|F_U - F_{U'}\|_\infty. 
\end{align*}
For inverted OCE risks, 
denote $\lambda_U \in \argmax_{\lambda \in [0,D]} \lambda + \mathbb{E}_{U}[\phi(\lambda-U)]$ and 
$\lambda_{U'} \in \argmax_{\lambda \in [0,D]} \lambda + \mathbb{E}_{U'}[\phi(\lambda-U)]$. 
The proof proceeds similarly by using the fact that 
$\rho_\ioce(F_{U'}) - \rho_\ioce(F_U) \leq \E_{U}[\phi(\lambda_{U'} -U)]- \E_{U'}[\phi(\lambda_{U'} -U')]$
and $\rho_\ioce(F_{U}) - \rho_\ioce(F_U') \leq \E_{U'}[\phi(\lambda_{U} -U')] - \E_{U}[\phi(\lambda_{U} -U)]$.
Following similar steps, we obtain that
\begin{align*}
    |\rho_\ioce(F_U) - \rho_\ioce(F_{U'})|
    \leq \max_{\lambda \in [0,D]} (\phi(x-D) - \phi(x))\|F_U - F_{U'}\|_\infty. 
\end{align*}
\end{proof}

Spectral risks (also known as $L$-risks or rank-weighted risks) 
are a subset of distortion risk functionals.
As noted in~\citet{bauerle2021minimizing},
a spectral risk can be written as a distortion risk (Equation~\eqref{eq:distortion-risk})
with the following distortion function: for $t \in [0,1]$, 
\begin{align*}
    g(t) = \int_0^t h(s) ds,
\end{align*}
where $h: [0,1] \to \mathbb{R}_+$ is the non-decreasing  spectrum function 
that integrates to $1$.  
Since $g$ is Lipschitz when $h$ is bounded (i.e., $g'(t) = h(t)$ for $t \in (0,1)$), 
spectral risks are Lipschitz on the space of bounded random variables
when their spectrum is bounded. %

Finally, 
examples of risk functionals that are H\"older 
but are not Lipschitz 
include distortion risks whose distortion functions 
are H\"older but not Lipschitz.

\clearpage

\section{Proof of Results in Section~\ref{Sec:SV}}\label{appendix:generalization_proofs}
We note that in the following proofs,
$Z$ is used to denote a generic random variable
and the loss function is denoted by 
$\ell: \X \times \Y \times \Y \rightarrow \Real$.
(The loss function presented in the main text is a special case of this.)
For a given $(X,Y, f(X))$, 
the loss is denoted by $\ell(X, Y, f(X))$.

    \subsection{Auxillary Lemmas}\label{appendix:auxillary}
The below two auxillary lemmas are mainly adaptations to  
the class note from Professor Roberto Imbuzeiro Oliveira~\citep{Oliveira_undated}.

\begin{lemma}\label{lemma:Roberto-max-trick-part-a}
Let $\GuncClass(1) := \mybrace{g(\cdot\; ; r):= \Ind_{\lbrace \cdot \; \leq r \rbrace}: \forall r \in \Real}$. 
For a fixed $f \in \FuncClass$, iid sample $\{X_i, Y_i\}_{i=1}^n$ with joint
probability measure $\Prob$, 
and a loss function $\ell:\X\times\Y\times\Y\rightarrow \Real$, 
we have that 
\begin{align}\label{eq:trick-part-a-1}
    &\sup_{g\in\GuncClass(1)} 
    \left|\sum_{i=1}^n \xi_{i} 
    g(\ell(X_{i},Y_{i},f(X_{i})))\right|=\max_{j \in [n]}
    \left|\sum_{i=1}^j \xi_{i}\right|,
\end{align}
and further 
\begin{align}\label{eq:trick-part-a-2}
    \E_{\Prob,\R}\left[\exp\left(\frac{\lambda}{n} 
    \sup_{g\in\GuncClass(1)} 
    \left|\sum_{i=1}^n\xi_i g(\ell(X_i,Y_i,f(X_i)))\right|\right)\right]
    \leq 2\E_{\R}\left[\max_{j \in [n]}
    \left( \exp\left(\frac{\lambda}{n}\sum_{i=1}^j\xi_i\right)
    \Ind_{\lbrace\sum_{i=1}^j \xi_i \geq 0\rbrace}\right)\right].
\end{align}
\end{lemma}

\begin{remark}
We note that an important property of Lemma~\ref{lemma:Roberto-max-trick-part-a}
is that the bound is independent of the samples $\lbrace X_i, Y_i\rbrace_{i=1}^n$.
\end{remark}

\begin{proof}
    Let $\{R_i\}_{i=1}^n$ denote the sorted sequence of 
    $\{\ell(X_i,Y_i,f(X_i))\}_{i=1}^n$, where $R_1\leq R_2 \ldots \leq R_n$. Using $\{R_i\}_{i=1}^n$, we have 

    \begin{align*}%
        \E_{\Prob,\R}\left[\exp\left(\frac{\lambda}{n} \sup_{g\in\GuncClass(1)} 
        \left|\sum_{i=1}^n\xi_i g(\ell(X_i,Y_i,f(X_i)))\right|\right)\right]
        = \E_{\Prob,\R}\left[\exp\left(\frac{\lambda}{n} 
        \sup_{g\in\GuncClass(1)} \left|\sum_{i=1}^n\xi_i g(R_i)\right|\right)\right]
    \end{align*}

    Consider a function $g(t; r) = \Ind_{\lbrace t \leq r\rbrace}$. 
    For such a function, $\sum_{i=1}^n \xi_i g(R_i; r)$ is equal to 
    \begin{itemize}
        \item $0$ if $r<\min_{i \in [n]} R_i $,
        \item $\sum_{i=1}^j\xi_i$ when $R_j\leq r <R_{j+1}$ for some $j\in\lbrace1,\ldots,n-1\rbrace$,
        \item $\sum_{i=1}^n\xi_i$ otherwise. 
    \end{itemize}
    
Therefore, we have that 
\begin{align*}%
    &\sup_{g\in\GuncClass(1)} 
    \left|\sum_{i=1}^n\xi_{i} g(\ell(X_{i},Y_{i},
    f(X_{i})))\right|
    =\max_{j\in[n]}\left|\sum_{i=1}^j\xi_{i}\right|.
\end{align*}

Finally, we notice that 
\begin{align}\label{eq:remove-absoluate-value-xi}%
    &\E_{\Prob,\R}\left[\sup_{g\in\GuncClass(1)}\exp\left(\frac{\lambda}{n}\left|
    \sum_{i=1}^n \xi_ig(R_i)\right|\right) \right] 
    =\E_{\Prob,\R}\left[ \max_{j \in [n]} 
    \exp\left(\frac{\lambda}{n}\left|\sum_{i=1}^j\xi_i\right|\right)\right]\nonumber \nonumber \\
    =& \E_{\Prob,\R}\Big[\max_{j \in [n]} \Big( \exp\left(\frac{\lambda}{n}\sum_{i=1}^j\xi_i\right)
    \Ind_{\lbrace\sum_{i=1}^j\xi_i\geq0\rbrace} 
    + \exp\left(-\frac{\lambda}{n}\sum{i=1}^j\xi_i\right)\Ind_{\lbrace\sum_{i=1}^j\xi_i<0\rbrace}\Big)\Big]\nonumber\\
    =& \E_{\Prob,\R}\left[\max_{j \in [n]}\left( \exp\left(\frac{\lambda}{n}\sum_{i=1}^j\xi_i\right)\Ind_{\lbrace\sum_{i=1}^j\xi_i\geq0\rbrace}\right)\right]
    +\E_{\Prob,\R}\left[\max_{j \in [n]}\left(\exp\left(-\frac{\lambda}{n}\sum_{i=1}^j\xi_i\right)\Ind_{\lbrace\sum_{i=1}^j\xi_i<0\rbrace}\right)\right]\nonumber\\
    \leq&2\E_{\Prob,\R}\left[\max_{j \in [n]}\left( \exp\left(\frac{\lambda}{n}\sum_{i=1}^j\xi_i\right)\Ind_{\lbrace\sum_{i=1}^j\xi_i\geq0\rbrace}\right)\right].
\end{align}

\end{proof}

\begin{lemma}%
\label{lemma:Roberto-max-trick-part-b} %
Let $\{Z_i\}_{i=1}^n$ taking values in $\mathcal{Z}$
denote independent samples drawn from $\Prob$ and $\weight: \mathcal{Z} \to \Real_+$.
For $n$ independent Rademacher random variables $\{\xi_i\}_{i=1}^n$, 
we have that
for all $\lambda \geq 0$, 
\begin{align*}
    \E_{\Prob, \R}\left[ 
            \exp\left(\max_{j \in [n]} \frac{\lambda}{n}\sum_{i=1}^j w(Z_i)\xi_i
        \right)
    \right]
    \leq 2 \E_{\Prob, \R}\left[
        \exp\left(\frac{\lambda}{n}\sum_{i=1}^n \weight(Z_i)\xi_i
        \right)\right].
\end{align*}
Further, if $\frac{1}{n}\sum_{i=1}^n \weight(Z_i) \xi_i$ is mean zero $\frac{\gamma^2}{n}$-subGaussian, then 
\begin{align*}
    \E_{\R}\left[ 
            \exp\left(\max_{j \in [n]} \frac{\lambda}{n}\sum_i^j\xi_i
        \right)
    \right]
    \leq 2\exp\left(\frac{\lambda^2 \gamma^2}{2n}\right).
\end{align*}
\end{lemma}

\begin{remark}\label{remark:roberto-trick}
We note that an immediate consequence of Lemma~\ref{lemma:Roberto-max-trick-part-b} is  
\begin{align*}
    &\E_{\Prob, \R}\left[ \max_{j \in [n]}
         \left(  \exp\left( \frac{\lambda}{n}\sum_{i=1}^j w(Z_i)\xi_i
        \right)
    \Ind_{ \lbrace  \sum_{i=1}^j w(Z_i)\xi_i\geq 0 \rbrace}\right) \right] \\
    \leq& \E_{\Prob, \R}\left[ 
            \exp\left(\max_{j \in [n]} \frac{\lambda}{n}\sum_{i=1}^j w(Z_i)\xi_i
        \right)\right]
    \leq 2 \E_{\Prob, \R}\left[
        \exp\left(\frac{\lambda}{n}\sum_{i=1}^n \weight(Z_i)\xi_i
        \right)\right].
\end{align*}
\end{remark}

\begin{proof}
The proof contains two main steps:\\
\textbf{Step 1:}
We will show that for all $t > 0$, 
\begin{align}\label{eq:step-1}
    \Prob \left( \max_{j \in [n]} \sum_{i=1}^j w(Z_i) \xi_i \geq t \right) \leq 
    2 \Prob\left(\sum_{i=1}^n w(Z_i) \xi_i \geq t \right).
\end{align}
To show~\eqref{eq:step-1}, for $t >0$, 
consider events  $E_0:=\emptyset$ and $E_j:=\lbrace\sum_{i=1}^j\weight(Z_i) \xi_i \geq t,
\sum_{i=1}^l \weight(Z_i) \xi_i< t, \forall l< j\rbrace$ for $j \in [n]$, 
which states that $j$ is the first index such that 
the partial sum $\sum_{i=1}^j\weight(Z_i) \xi_i$ is at least $t$.  
We first notice that 
\begin{align*}
    \mybrace{\max_{j}\sum_{i=1}^j\weight(Z_i)\xi_i\geq t} \subset\bigcup_{j=0}^n E_j.
\end{align*}
Since $E_j$ and $\sum_{i=j+1}^n w(Z_i) \xi_i \geq 0$ 
implies that $\sum_{i=1}^n w(Z_i)\xi_i \geq t$, we obtain 
\begin{align}\label{eq:step-1-key}
 \bigcup_{j=0}^n \left(E_j\bigcap \mybrace{\sum_{i=j+1}^n \weight(Z_i)\xi_i
 \geq 0}\right) \subset \mybrace{\sum_{i=1}^n \weight(Z_i)\xi_i\geq t}.
\end{align}
Moreover, for any $j \in [n]$, we have  
\begin{align}\label{eq:step-1-c}
    \Prob\left(\sum_{i=j+1}^n \weight(Z_i)\xi_i\geq 0\right)\geq \frac{1}{2},
\end{align}
since $\sum_{i=j+1}^n \weight(Z_i)\xi_i$ is symmetic around $0$ 
(i.e., $\sum_{i=j+1}^n \weight(Z_i)\xi_i \stackrel{d}{=} - \sum_{i=j+1}^n \weight(Z_i)\xi_i $)
for all $i \in \{0, \ldots, n\}$.
For any $j \in [n]$, since the event  $E_j$ (dependeing on $\{w(Z_i), \xi_i\}_{i\leq j}$) is independent of 
$\lbrace \sum_{i=j+1}^n \weight(Z_i)\xi_i \geq 0 \rbrace$, 
we have 
\begin{align*}
\Prob\left(E_j\bigcap \mybrace{\sum_{i=j+1}^n \weight(Z_i) \xi_i\geq 0}\right) 
=\Prob\left(E_j\right)\Prob\left(\sum_{i=j+1}^n
\weight(Z_i) \xi_i\geq 0\right) 
\geq \frac{\Prob\left(E_j\right)}{2}. 
\end{align*} 
As a result, we have 
\begin{align*}
    \Prob \left(\sum_{i=1}^n \weight(Z_i)\xi_i\geq t\right)
    &\geq \Prob \left( \bigcup_{j=0}^n 
    \left(E_j\bigcap \mybrace{\sum_{i=j+1}^n\weight(Z_i)\xi_i\geq 0}\right)\right)\\
    &=\sum_{j=0}^n \Prob \left( E_j\bigcap \mybrace{\sum_{i=j+1}^n\weight(Z_i)\xi_i\geq 0} \right)\\
    &\geq \sum_{j=1}^n \frac{\Prob\left(E_j\right)}{2} \\
    &\geq \frac{1}{2}\Prob\left(\max_{j \in [n]}\sum_{i=1}^j\weight(Z_i)\xi_i\geq t\right),
\end{align*}
where the first equality holds because for $i \neq j$, $E_i \cap E_j = \emptyset$, 
the first inequality follows from~\eqref{eq:step-1-key}
and the last inequality comes from the union bound.

\textbf{Step 2:}
For any differentiable $f$ and random variable $X$, 
\begin{align}\label{eq:step-2}
    \E[f(X) \Ind_{\lbrace X \geq 0\rbrace}]
    &= \E\left[ \left( f(0) + \int_0^X f'(t) dt\right)\right]\nonumber\\
    &= f(0) \E[\Ind_{\lbrace X \geq 0\rbrace}] 
        + \E\left[\int_0^\infty f'(t) \Ind_{\lbrace X \geq t \rbrace} dt\right]\nonumber\\
    &= f(0) \Prob( X \geq 0 ) + \int_0^\infty f'(t) \Prob(X \geq t) dt,
\end{align}
where the last equality follows from Fubini's theorem. 
Putting it altogether, we have 
\begin{align*}
    &\E_{\Prob, \R}\left[ 
            \exp\left(\max_{j \in [n]} \frac{\lambda}{n}\sum_{i=1}^j w(Z_i)\xi_i
        \right)\right]\\
=&1 
+ \lambda \int_0^\infty \exp\left(\lambda t \right) 
\Prob\left(  \max_{j \in [n]} \frac{1}{n} \sum_{i=1}^j w(Z_i) \xi_i \geq t\right) dt\\
\leq&1
+ 2\lambda \int_0^\infty \exp\left(\lambda t \right) 
\Prob\left( \frac{1}{n} \sum_{i=1}^n w(Z_i) \xi_i \geq t\right) dt\\
=&1
+ 2 \E_{\Prob, \R} \left[\left( \exp\left(\frac{\lambda}{n} \sum_{i=1}^n w(Z_i) \xi_i \right) - 1\right)
\Ind_{\lbrace \sum_{i=1}^n w(Z_i)\xi_i\geq 0 \rbrace}
\right]\\
=& 1
+ 2 \E_{\Prob, \R} \left[ \exp\left(\frac{\lambda}{n} \sum_{i=1}^n w(Z_i) \xi_i \right) 
\right] - 2 \Prob\left( \sum_{i=1}^n w(Z_i)\xi_i\geq 0 \right)\\
\leq& 2 \E_{\Prob, \R} \left[ \exp\left(\frac{\lambda}{n} \sum_{i=1}^n w(Z_i) \xi_i \right) 
\right],
\end{align*}
where 
the first inequality follows from~\eqref{eq:step-1}, 
the second equality uses~\eqref{eq:step-2} with $f(t) = e^{\lambda t}$
and $X = \max_{j \in [n]} \frac{1}{n} \sum_{i=1}^j w(Z_i) \xi_i$,
and the last inequality follows from~\eqref{eq:step-1-c}.  
Finally, if $\frac{1}{n}\sum_{i=1}^n \weight(Z_i) \xi_i$ is mean zero $\frac{\gamma^2}{n}$-subGaussian, then
using the definition of a subGaussian random variable, we obtain  
\begin{align*}
    2 \E_{\Prob, \R} \left[ \exp\left(\frac{\lambda}{n} \sum_{i=1}^n w(Z_i) \xi_i \right) 
\right] 
    \leq  2\exp\left(\frac{\lambda^2 \gamma^2}{2n}\right).
\end{align*}
\end{proof}

    \subsection{Proof of Theorem~\ref{thm:SV_Gen}}

\thmSvGen*

\begin{proof}\label{proof:Thm_SV_Gen} 

We first analyze the sensitivity of the sup-norm of the \CDF estimator over 
$\FuncClass$ and $\GuncClass(1)$. For a given two sets $\lbrace x_i,y_i\rbrace_{i=1}^n$ and $\lbrace x'_i,y'_i\rbrace_{i=1}^n$, which just differ in $j$'th entry, let

\begin{equation*}
    \wh{F}_1(r;f) := \frac{1}{n}\sum_{i=1}^n \mathbbm{1}_{\{\ell(y_i,f(x_i)) \leq r\}}~\text{and},~ \wh{F}_2(r;f) := \frac{1}{n}\sum_{i=1}^n \mathbbm{1}_{\{\ell(y'_i,f(x'_i)) \leq r\}}
\end{equation*}

Then, if $\sup_{f\in\FuncClass}\sup_{r\in\Real} \left|\wh{ F}_1(r;f) -  F(r;f)\right|\geq\sup_{f\in\FuncClass}\sup_{r\in\Real} \left|\wh{ F}_2(r;f) -  F(r;f)\right|$, we have 

\begin{align*}
    &\sup_{f\in\FuncClass}\sup_{r\in\Real} \left|\wh{ F}_1(r;f) -  F(r;f)\right|-\sup_{f\in\FuncClass}\sup_{r\in\Real} \left|\wh{ F}_2(r;f) -  F(r;f)\right|\\
    &=\sup_{f\in\FuncClass}\sup_{r\in\Real} \left|\frac{1}{n}\sum_{i=1}^n \mathbbm{1}_{\{\ell(X_i,y_i,f(x_i)) \leq r\}} -  F(r;f)\right|-\sup_{f\in\FuncClass}\sup_{r\in\Real} \left|\frac{1}{n}\sum_{i=1}^n \mathbbm{1}_{\{\ell(x_i,y'_i,f(x'_i)) \leq r\}} -  F(r;f)\right|\\
    &=\sup_{f\in\FuncClass}\sup_{r\in\Real} \left|\frac{1}{n}\sum_{i=1}^n \mathbbm{1}_{\{\ell(x'_i,y'_i,f(x'_i)) \leq r\}} +\frac{1}{n}\left( \mathbbm{1}_{\{\ell(x_j,y_j,f(x_j)) \leq r\}}-\mathbbm{1}_{\{\ell(x'_j,y'_j,f(x'_j)) \leq r\}}\right) -  F(r;f)\right|\\
    &\quad\quad\quad\quad\quad\quad\quad\quad\quad\quad\quad\quad\quad\quad\quad\quad\quad\quad\quad\quad-\sup_{f\in\FuncClass}\sup_{r\in\Real} \left|\frac{1}{n}\sum_{i=1}^n \mathbbm{1}_{\{\ell(x'_i,y'_i,f(x'_i)) \leq r\}} -  F(r;f)\right|\\
    &=\sup_{f\in\FuncClass}\sup_{r\in\Real} \frac{1}{n}\left| \mathbbm{1}_{\{\ell(x_j,y_j,f(x_j)) \leq r\}}-\mathbbm{1}_{\{\ell(x'_j,y'_j,f(x'_j)) \leq r\}} \right|\leq \frac{1}{n}.
\end{align*}
\begin{align*}
    &\left|\sup_{f\in\FuncClass}\sup_{r\in\Real} \left|\wh{ F}_1(r;f) -  F(r;f)\right|-\sup_{f\in\FuncClass}\sup_{r\in\Real} \left|\wh{ F}_2(r;f) -  F(r;f)\right|\right|\\
    &\quad\quad\quad\leq \sup_{f\in\FuncClass}\sup_{r\in\Real} \frac{1}{n}\left| \mathbbm{1}_{\{\ell(x_j,y_j,f(x_j)) \leq r\}}-\mathbbm{1}_{\{\ell(x'_j,y'_j,f(x'_j)) \leq r\}} \right|=\frac{1}{n}
\end{align*}
This bound holds no matter what $j$ and what data set we choose.  
Using bounded difference inequality, i.e.,  McDiarmid's inequality ~\citep{boucheron2013concentration}, we have, 

\begin{align}\label{eq:McDiarmid}
    \Prob\left(\sup_{f\in\FuncClass}\sup_{r\in\Real} \left|\wh{ F}(r;f) -  F(r;f)\right|-
    \E_\Prob\left[\sup_{f\in\FuncClass}\sup_{r\in\Real} \left|\wh{ F}(r;f) -  F(r;f)\right|\right]
\leq \sqrt{\frac{\log(\frac{1}{\delta})}{2n}}\right)
\end{align}
with probability at least $1-\delta$. Using a ghost sample set $\lbrace X'_i,Y'_i\rbrace_{i=1}^n$ we have 
\allowdisplaybreaks
\begin{align}\label{eq:supsup}
    &\E_\Prob\left[\sup_{f\in\FuncClass}\sup_{r\in\Real} \left|\wh{ F}(r;f) -  F(r;f)\right|\right]\nonumber\\
    =&\E_\Prob\left[\sup_{f\in\FuncClass}\sup_{r\in\Real} \left|\frac{1}{n}\sum_{i=1}^n \mathbbm{1}_{\{\ell(X_i,Y_i,f(X_i)) \leq r\}}  -  
    \E_\Prob\left[\frac{1}{n}\sum_{i=1}^n \mathbbm{1}_{\{\ell(X'_i,Y'_i,f(X'_i)) \leq r\}}\right]
    \right|\right]\nonumber\\
    =&\E_\Prob\left[\sup_{f\in\FuncClass}\sup_{g\in\GuncClass(1)} \left|\frac{1}{n}\sum_{i=1}^n g(\ell(X_i,Y_i,f(X_i)))  -  
    \E_\Prob\left[\frac{1}{n}\sum_{i=1}^n g(\ell(X'_i,Y'_i,f(X'_i)))\right]
    \right|\right]\nonumber\\
    =&\E_\Prob\left[\sup_{f\in\FuncClass}\sup_{g\in\GuncClass(1)} \left|\frac{1}{n}\E_\Prob\left[\sum_{i=1}^n g(\ell(X_i,Y_i,f(X_i)))  -  
    \sum_{i=1}^n g(\ell(X'_i,Y'_i,f(X'_i)))\Big|\sigma\left(\lbrace X_i,Y_i\rbrace_{i=1}^n\right)\right]
    \right|\right]\nonumber\\
    \leq&\E_\Prob\left[\sup_{f\in\FuncClass}\sup_{g\in\GuncClass(1)} \left|\frac{1}{n}\sum_{i=1}^n\left( g(\ell(X_i,Y_i,f(X_i)))  -  
     g(\ell(X'_i,Y'_i,f(X'_i)))\right)
    \right|\right]\nonumber\\
    \leq&\E_{\Prob,\R}\left[\sup_{f\in\FuncClass}\sup_{g\in\GuncClass(1)} \left|\frac{1}{n}\sum_{i=1}^n\xi_i\left( g(\ell(X_i,Y_i,f(X_i)))  -  
     g(\ell(X'_i,Y'_i,f(X'_i)))\right)
    \right|\right]\nonumber\\
    \leq&2\E_{\Prob,\R}\left[\sup_{f\in\FuncClass}\sup_{g\in\GuncClass(1)} \left|\frac{1}{n}\sum_{i=1}^n\xi_i g(\ell(X_i,Y_i,f(X_i)))  
    \right|\right]=2\Rade(n,\FuncClass).
\end{align}
Putting \eqref{eq:McDiarmid} and \eqref{eq:supsup} together concludes the proof. 
\end{proof}

    \subsection{Permuation Complexity}

We note that this proof, along with many other proofs for Section~\ref{Sec:SV},
is based on the machinery and techniques in Massart's finite class Lemma~\citep{massart2000some} and DKW inequality in \citep{devroye2013probabilistic}. 

\thmPermutationComplexity*

\begin{proof}
    For a positive $\lambda$, we have,
    \begin{align}\label{eq:RadeExpGen}
        \exp\left(\lambda\Rade(n,\FuncClass)\right)&=\exp\left(\frac{\lambda}{n} \E_{\Prob,\R}\left[\sup_{f\in\FuncClass}\sup_{g\in\GuncClass(1)} \left|\sum_{i=1}^n\xi_i g(\ell(X_i,Y_i,f(X_i)))\right|\right]\right)\nonumber\\
        &\leq \E_{\Prob,\R}\left[\exp\left(\frac{\lambda}{n} \sup_{f\in\FuncClass}\sup_{g\in\GuncClass(1)} \left|\sum_{i=1}^n\xi_i g(\ell(X_i,Y_i,f(X_i)))\right|\right)\right]\nonumber\\
        &\leq \E_{\Prob,\R}\left[\E_{\Prob,\R}\left[\exp\left(\frac{\lambda}{n} \sup_{f\in\FuncClass}\sup_{g\in\GuncClass(1)} \left|\sum_{i=1}^n\xi_i g(\ell(X_i,Y_i,f(X_i)))\right|\right)\Big|\sigma\left(\lbrace X_i,Y_i\rbrace_{i=1}^n\right)\right]\right]
    \end{align}

    For any $f\in\FuncClass$, let $\pi_f$ denote a permutation such that  $\ell(X_{\pi_f(i)},Y_{\pi_f(i)},f(X_{\pi_f(i)}))\leq \ell(X_{\pi_f(j)},Y_{\pi_f(j)},f(X_{\pi_f(j)}))$ for any $i,j\in[n]$ and $i\leq j$. Therefore we have,
    
    \begin{align*}
     \sup_{g\in\GuncClass(1)} \left|\sum_{i=1}^n\xi_i g(\ell(X_i,Y_i,f(X_i)))\right|= \sup_{g\in\GuncClass(1)} \left|\sum_{i=1}^n\xi_{\pi_f(i)} g(\ell(X_{\pi_f(i)},Y_{\pi_f(i)},f(X_{\pi_f(i)})))\right|
    \end{align*}

    Consider a function $g(r') =\Ind_{\lbrace r'\leq r\rbrace}$. For such a function, $\sum_{i=1}^n\xi_{\pi_f(i)} g(\ell(X_{\pi_f(i)},Y_{\pi_f(i)},f(X_{\pi_f(i)})))$ is equal to,
    
    \begin{itemize}
        \item $0$ if $r<\min_i\lbrace \ell(X_{\pi_f(i)},Y_{\pi_f(i)},f(X_{\pi_f(i)}))\rbrace_i^n$,
        \item $\sum_i^j\xi_{\pi_f(i)}$ when $\ell(X_{\pi_f(j)},Y_{\pi_f(j)},f(X_{\pi_f(j)}))\leq r<\ell(X_{\pi_f(j+1)},Y_{\pi_f(j+1)},f(X_{\pi_f(j+1)}))$ for a $j\in\lbrace1,\ldots,n-1\rbrace$,
        \item $\sum_i^n\xi_{\pi_f(i)}$ otherwise. 
    \end{itemize}
    
    Using this property, we have,
    \begin{align}\label{eq:PermutationComp_Cramer}
        &\sup_{g\in\GuncClass(1)} \left|\sum_{i=1}^n\xi_{\pi_f(i)} g(\ell(X_{\pi_f(i)},Y_{\pi_f(i)},f(X_{\pi_f(i)})))\right|=\max_j\left|\sum_i^j\xi_{\pi_f(i)}\right|
    \end{align}

    Using this equality, we can further extend the Eq.~\ref{eq:RadeExpGen},

    \begin{align}\label{eq:RadeExpGen_Permut}
        \exp\left(\lambda\Rade(n,\FuncClass)\right)
        &\leq \E_{\Prob,\R}\left[\E_{\Prob,\R}\left[\exp\left(\frac{\lambda}{n} \sup_{f\in\FuncClass}\max_j\left|\sum_i^j\xi_{\pi_f(i)}\right|\right)\Big|\sigma\left(\lbrace X_i,Y_i\rbrace_{i=1}^n\right)\right]\right]\nonumber\\
         &\leq \E_{\Prob,\R}\left[\E_{\Prob,\R}\left[\N_\Pi(\FuncClass_\ell,\lbrace X_i,Y_i\rbrace_{i=1}^n)\exp\left(\frac{\lambda}{n} \max_j\left|\sum_i^j\xi_{i}\right|\right)\Big|\sigma\left(\lbrace X_i,Y_i\rbrace_{i=1}^n\right)\right]\right]\nonumber\\
         &\leq \N_\Pi(n,\FuncClass_\ell,\Prob)\E_{\Prob,\R}\left[\exp\left(\frac{\lambda}{n} \max_j\left|\sum_i^j\xi_{i}\right|\right)\right],
    \end{align}
    where the second inequality follows from the fact that effectively, 
    there are at most $\N_\Pi(\FuncClass_\ell,\lbrace X_i,Y_i\rbrace_{i=1}^n)$ number of $\pi_f$'s.
    Using the same derivation for~\eqref{eq:remove-absoluate-value-xi} (Lemma~\ref{lemma:Roberto-max-trick-part-a}), we have 
    \begin{align*}
        \exp\left(\lambda\Rade(n,\FuncClass)\right)
        &\leq 2\N_\Pi(n,\FuncClass_\ell,\Prob)\E_{\Prob,\R}\left[\max_{j}\left( \exp\left(\frac{\lambda}{n}\sum_i^j\xi_i\right)\Ind_{\lbrace\sum_i^j\xi_i\geq0\rbrace}\right)\right].
    \end{align*}
    By Lemma~\ref{lemma:Roberto-max-trick-part-b}, we have 
    \begin{align*}
        \exp\left(\lambda\Rade(n,\FuncClass)\right)
        &\leq 4\N_\Pi(n,\FuncClass_\ell,\Prob)\exp\left(\frac{\lambda^2}{2n}\right)
    \end{align*}
    
    Now, taking the log from both sides, and dividing by $\lambda$, we have 
    $
        \Rade(n,\FuncClass)
        \leq \frac{\log( 4\N_\Pi(n,\FuncClass_\ell,\Prob))}{\lambda}+\frac{\lambda}{2n}. 
    $
    Choosing $\lambda=\sqrt{2n\log(2\N_\Pi(n,\FuncClass_\ell,\Prob))}$, we obtain the final result    
    \begin{align*}
        \Rade(n,\FuncClass)
        &\leq \sqrt{\frac{\log( 4\N_\Pi(n,\FuncClass_\ell,\Prob))}{2n}}.
    \end{align*}
\end{proof}

\thmFiniteFuncClass*

\begin{proof}\label{proof:finite-Rade} 
The result follows since when $|\FuncClass| < \infty$, 
$\N_\Pi(n,\FuncClass_\ell,\Prob) \leq |\FuncClass|$, 
i.e., we need at most one permutation function 
to sort the losses for each $f \in \FuncClass$. 

\end{proof}

\clearpage
\section{Proof of Results in Section~\ref{sec:application-risk-assessement}}
\label{appendix:risk-assessment}

\thmRiskAssessment*

\begin{proof}
For all $\rho \in \mathbb{T}$, 
since  
$\rho$ is $L(\rho, 1, \L_\infty)$ H\"older,
we have that 
$|\rho(F(\cdot; f)) - \rho(\wh F(\cdot; f))| \leq L(\rho, 1, \L_\infty) \|F - \wh F\|_\infty$. 
The desired result then follows. 
\end{proof}

{
{The error of risk assessment can be bounded using distances other than the sup-norm, such as the Wasserstein distance.}
For two random variables $U$ and $U'$ 
with bounded support $[0, D]$, 
the dual form of the Wasserstein distance $W_1(F_U, F_{U'})$ 
is given by~\citep{vallender1974calculation},
\begin{align*}%
    W_1(F_U, F_{U'}) = \int_0^D |F_{U}(t)-F_{U'}(t)|dt\leq D\|F_U-F_{U}\|_{\infty}.
\end{align*}
This inequality suggests the following corollary. 
\begin{restatable}{corollary}{corRiskAssessment}\label{corollary:risk-assessment-w1}
Under the setting of Theorem~\ref{Thm:riskbound} where the loss has support $[0,D]$, 
with probability $1-\delta$,  for all $\rho \in \mathbb{T}$, 
we have 
\begin{align*}
     \sup_{f \in \FuncClass}|\rho(F(\cdot;f))-\rho(\wh F(\cdot;f))|\leq   L(\rho,p,W_1)  D^p \epsilon^p,
\end{align*}
where $\mathbb{T}$ is the set of $L(\cdot, p, W_1)$ H\"older risk functionals 
on the space of bounded random variables. 
\end{restatable}
}

\begin{proof}
    For all $\rho \in \mathbb{T}$, 
    since  
    $\rho$ is $L(\rho, p, W_1)$ H\"older,
    we have that 
    $|\rho(F(\cdot; f)) - \rho(\wh F(\cdot; f))| \leq L(\rho, p, W_1) W_1(F, \wh F)^p
    \leq L(\rho, p, W_1) D^p \|F - \wh F\|_\infty^p$. 
    The desired result then follows. 
\end{proof}

\corDistortionRiskAssessment*

\begin{proof}
    As is shown in~\citet[Lemma 4.1]{huang2021off}, 
    $\rho$ is $L$-Lipschitz. 
    Thus, directly applying Theorem~\ref{Thm:riskbound} 
    concludes the proof.
\end{proof}

\clearpage

\section{Proofs for results in Section~\ref{sec:optimization}}
\label{appendix:optimization}

\subsection{Proof of Lemma~\ref{thm:almost-differentiable}}
Before proving Lemma~\ref{thm:almost-differentiable}, 
we first present two auxiliary lemmas.

\begin{lemma}\label{lemma:min_max_cnt}
    For any continuous function 
        $g,h: \mathbb{R}^d \to \mathbb{R}$, 
        $\min(g, h)$ and $\max(g, h)$ are continuous.
    Similarly, for any Lipschitz continuous function  
    $g,h: \mathbb{R}^d \to \mathbb{R}$, 
    $\min(g, h)$ and $\max(g, h)$ are  Lipschitz continuous.   
    \end{lemma}
    
    \begin{proof}
    Denote $f_{\min} = \min(g,h)$
    and $f_{\max} = \max(g,h)$. 
    \\
    \textbf{Continuity:}
    We first consider the case when both $g$ and $h$ are continuous. 
    Define $\overline{\Theta} = \{\theta \in \mathbb{R}^d: 
    g(\theta) = h({\theta})\}$. 
    For $\theta \notin \overline{\Theta}$, 
    $f_{\min}$ and $f_{\max}$ are continuous 
    since $g, h$ are continuous. 
    Consider $\theta \in \overline{\Theta}$. 
    For every $\epsilon > 0$, there exists
    $\delta_g, \delta_h > 0$ such that  for all $\theta' \in \mathbb{R}^d$, 
    $\|\theta - \theta'\|\leq \epsilon \implies 
    |g(\theta) - g(\theta')| \leq \delta_g, 
    |h(\theta) - h(\theta')| \leq \delta_h$.
    In addition, $f_{\min}(\theta) = g(\theta) = h(\theta)$, 
    $f_{\max}(\theta) = g(\theta) = h(\theta)$
    and $f_{\min}(\theta')$, $f_{\max}(\theta')$ 
    can be either $g(\theta')$ or $h(\theta')$.
    Combining both facts gives us that
    $|f_{\min}(\theta) - f_{\min}(\theta')| \leq \max(\delta_g, \delta_h)$
    and $|f_{\max}(\theta) - f_{\max}(\theta')| \leq \max(\delta_g, \delta_h)$.
    
    \textbf{Lipschitz Continuity:} 
    We next work with the case where both $g$ and $h$ are Lipschitz continuous. 
    When  
    $
        |f_{\max}(\theta) - f_{\max}(\theta')|
        = |g(\theta) - h(\theta')|,
    $
    we have the following two cases:
    \begin{enumerate}
        \item $g(\theta) > h(\theta')$: 
        $|g(\theta) - h(\theta')| = g(\theta) - h(\theta')
        \leq g(\theta) - g(\theta'), %
        $
        since $f_{\max}(\theta') = h(\theta')$. 
        \item $g(\theta) \leq h(\theta')$: $|g(\theta) - h(\theta')| = h(\theta') - g(\theta)
        \leq h(\theta') - h(\theta),
        $ since $ f_{\max}(\theta) = g(\theta)$. 
    \end{enumerate}
    Since both $h$ and $g$ are Lipschitz continuous, 
    we obtain that 
    $$|f_{\max}(\theta) - f_{\max} (\theta') | 
    {\leq} \max_{f \in \{g, h\}} |f(\theta) - f(\theta')| 
    \lesssim \|\theta - \theta'\|,$$
    showing that $f_{\max}$ is Lipschitz continuous. 
    The proof completes with the fact that $\min(f,g) = -\max(-f, -g)$. 
\end{proof}

\begin{lemma}\label{lemma:order-statistitics-continuous}
    If $\{\ell_\theta(x_j, y_j)\}_{j=1}^n$ are Lipschitz continuous %
    in $\theta$, 
    then for all $i \in [n]$, 
    $\ell_\theta(\pi_\theta(i))$, i.e., the $i$-th smallest loss evaluated using data points $\{x_j, y_j\}_{j=1}^n$, 
    is 
    Lipschitz continuous %
    in $\theta$.
    \end{lemma}
    
    \begin{proof}
    The key observation is that the $i$-th smallest loss can be defined as 
    $\ell_\theta(\pi_\theta(i))
    = \min\{ \max\{\ell_\theta(j): j \in J\}: J \subseteq [n], |J| = i\}.
    $
    Since each $\ell_\theta(j)$ is 
    Lipschitz continuous %
    in $\theta$
    and that for $j' \in J$, 
    $\max\{\ell_\theta(j): j \in J\} = \max\{\ell_\theta(j'), \max\{\ell_\theta(j): j \in J \setminus \{j'\} \} \}$, 
    by Lemma~\ref{lemma:min_max_cnt}, 
    we have $\max\{\ell_\theta(j): j \in J\}$
    to be 
    Lipschitz continuous %
    in $\theta$.
    Similarly, since $\max\{\ell_\theta(j): j \in J\}$ is 
    Lipschitz continuous %
    in $\theta$, 
    we have $\min\{ \max\{\ell_\theta(j): j \in J\}: J \subseteq [n], |J| = i\}$ to be 
    Lipschitz continuous. %
    \end{proof}

\thmDifferentiableAE*

\begin{proof}
    Using Lemma~\ref{lemma:order-statistitics-continuous}, 
    we obtain that $\ell_\theta(\pi_\theta(i))$ is Lipschitz in $\theta \in \Theta$. 
    Following from a classical result of Rademacher~\citep[Theorem 9.60]{rockafellar2009variational}, i.e., a locally Lipschitz function is differentiable almost everywhere, 
    we have that $\rho(\wh F_\theta)$ is differentiable almost everywhere.
\end{proof}

\subsection{Local Convergence}

The proofs for Corollary~\ref{lemma:biased_stochastic_gradient} 
are standard~\citep{bottou2018optimization}, 
which we provide for completeness. 

\corEDRM* 
    
\begin{proof}
  For notation simplicity, 
  we use $h(\theta)$ to denote $\rho(\wh F_\theta)$
  and $g_t$ to denote $\nabla_\theta \rho(\wh F_\theta)$ when $\theta = \theta_t$. 
  Since $h(\theta)$ is differentiable almost everywhere, 
  following~\eqref{eq:gd-step}, 
  the sequence $\{h(\theta_t)\}_{t=1}^T$
  will be differentiable almost surely. 
  Since $h$ is $\beta$-smooth, we have that
  \begin{align*}
  h(\theta_{t+1})-h(\theta_t)\leq \nabla_\theta h(\theta_t)^\top(\theta_{t+1}-\theta_t)  +\frac{\beta}{2} \|\theta_{t+1}-\theta_t\|^2.      
  \end{align*}
  We denote the filtration 
  for the stochastic process ${\{\theta_t\}}_{t=1}^T$
  to be ${\{\F_t\}}_{t=1}^T$.
  Extending the above inequality, we have 
  \begin{align*}
  h(\theta_{t+1})-h(\theta_t)\leq \nabla_\theta h(\theta_t)^\top(-\eta (g_t + w_t))  +\frac{\beta}{2} \|-\eta (g_t + w_t)\|^2,      
  \end{align*}
  which suggests that 
  \begin{align*}
      \mathbb{E}[h(\theta_{t+1})-h(\theta_t)] \leq -\eta \nabla_\theta h(\theta_t)^\top g_t +\eta^2 \frac{\beta}{2}  (\|g_t\|^2 + d \sigma_w^2),
  \end{align*}
  where $\sigma_w^2 = 1/d$ is the variance of $w_t$.
  Therefore, for the conditional expectation, we have 
  \begin{align*}
  \E\left[h(\theta_{t+1})-h(\theta_t)\Big|\F_t\right] &\leq \E\left[-\eta \nabla_\theta h(\theta_t)^\top g_t +\eta^2\frac{\beta}{2}  (\|g_t\|^2 + 1) \Big|\F_k\right]\\
  &= -(\eta-\eta^2\frac{\beta}{2}) \mathbb{E}\left[\|\nabla_\theta h(\theta_t)\|^2\Big|\F_t\right]
  + \eta^2\frac{\beta}{2}. 
  \end{align*}
  
Using the telescoping sum and law of total expectation, we obtain 
  \begin{align}\label{eq:telescoping}
    \E\left[h(\theta_{T})-h(\theta_1)\Big|\F_1\right]
    &=\E\left[\sum_{t=1}^T h(\theta_{1})-h(\theta_{t-1})\Big|\F_1\right] \leq (\eta-\eta^2\frac{\beta}{2})\E\left[\sum_{t=1}^T- \|\nabla_\theta h(\theta_t)\|^2\Big|\F_1\right]+ T\eta^2 \frac{\beta}{2}. 
  \end{align}
  
  Use the fact that $h(\theta_\star) \leq  h(\theta_{T})$, 
  we have $\E\left[h(\theta_1)-h(\theta_{T})\right]\leq h(\theta_1)-h(\theta_\star)$, which implies %
  \begin{align*}
    (\eta-\eta^2\frac{\beta}{2})\E\left[\sum_{t=1}^T \|\nabla_\theta h(\theta_t)\|^2\right]\leq h(\theta_1)-h(\theta_\star) + T \eta^2 \frac{\beta}{2}. 
  \end{align*}
  
  Plugging the learning rate $\eta={\frac{1}{\beta \sqrt{T}}}$, we have
 $
      \eta-\eta^2\frac{\beta}{2} = {\frac{1}{\beta \sqrt{T}}} - {\frac{1}{2\beta {T}}}\geq {\frac{1}{\beta \sqrt{T}}} - {\frac{1}{2\beta \sqrt{T}}}\geq {\frac{1}{2\beta \sqrt{T}}}>0.
  $
  Therefore we have 
  \begin{align*}
    \frac{1}{2\beta \sqrt{T}}\E\left[\sum_{t=1}^T \|\nabla_\theta h(\theta_t)\|^2\right]\leq h(\theta_1)-h(\theta_\star) 
    + \frac{1}{2\beta}.
  \end{align*}
  Rearranging the above inequality gives the result.
\end{proof}

\clearpage

\section{Additional Experimental Details}
\label{appendix:experiment}

\subsection{Risk Assessment on ImageNet Models}
Figure~\ref{fig:cvar-evaluation} shows the CVaR 
of models presented in Table~\ref{tbl:risk-assessments} 
under different $\alpha$'s. 
We note that $\CVaR_\alpha$ is the expected value 
above the top $100\alpha$ percent losses. %

\begin{figure}[h]
\centering
\includegraphics[width=.45\linewidth]{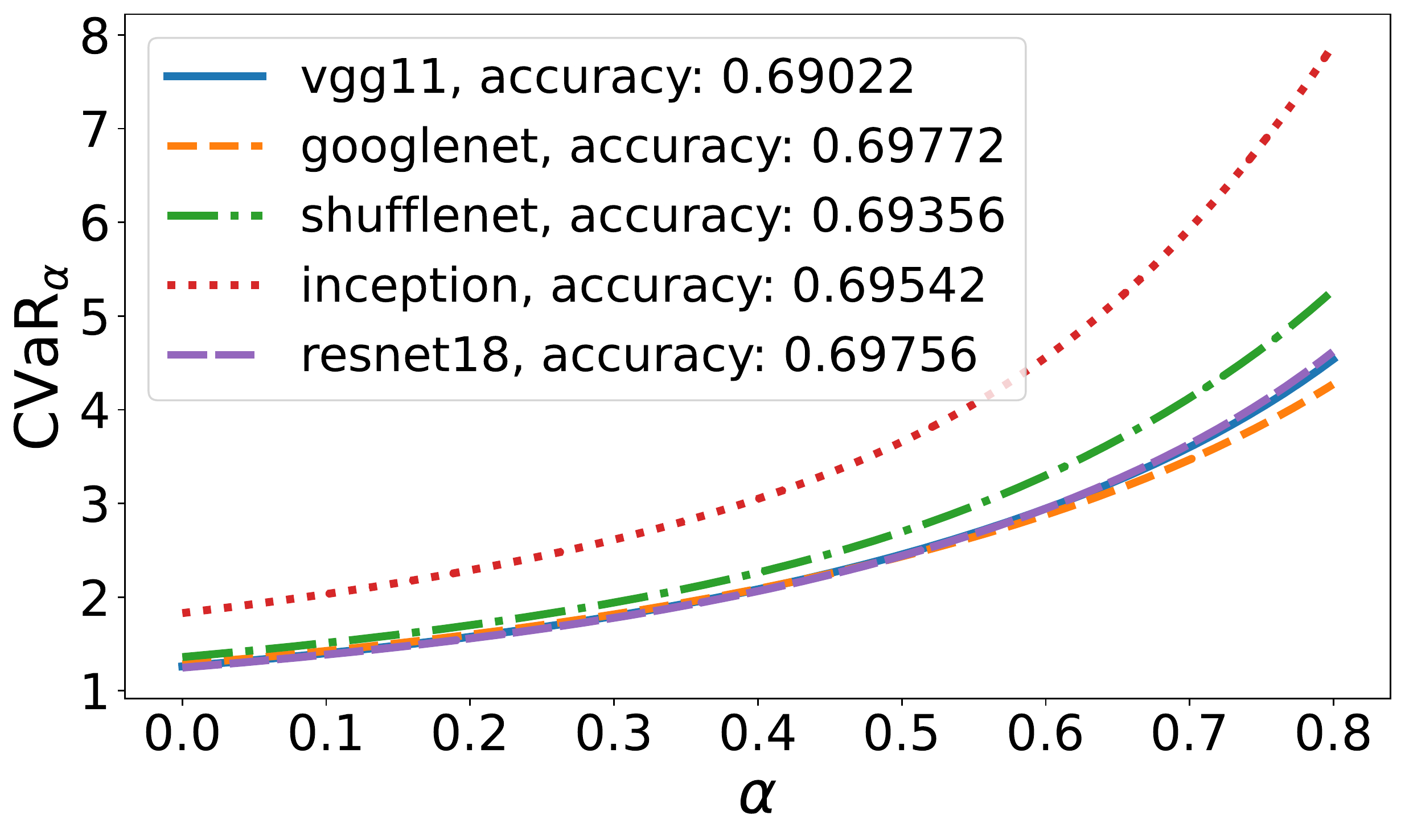}
\caption{$\CVaR_\alpha(\ell_f(Z))$ across different $\alpha$'s where 
$\ell_f$ is the cross-entropy loss evaluated on the ImageNet validation dataset.}
\label{fig:cvar-evaluation}
\end{figure}

\subsection{Empirical Distortion Risk Minimization}

\paragraph{Toy Example}
The data used in this experiment is generated using the \texttt{make\_blobs} function from \texttt{sklearn.datasets}
with the following parameters:
\texttt{n\_samples = [1000, 50]},
\texttt{centers = [[0.0, 0.0], [1.0, 1.0]]},
\texttt{cluster\_std = [1.5, 0.5]},
\texttt{random\_state = 0},
\texttt{shuffle = False}.

\end{document}